\documentclass[11pt]{article}
\usepackage{fullpage}
\usepackage{amsmath,amsfonts,amsthm,amssymb}
\usepackage{url}
\usepackage{algorithm}
\usepackage{algorithmic}
\usepackage{bbm}
\usepackage{float}
\usepackage{framed}
\usepackage{enumerate}
\usepackage{color}
\usepackage[usenames,dvipsnames,svgnames,table]{xcolor}
\usepackage[colorlinks=true, linkcolor=red, urlcolor=blue, citecolor=gray]{hyperref}
\usepackage{array}
\usepackage[para]{footmisc}
\usepackage{epsfig}
\usepackage{geometry}
\makeatletter
\newcommand{\thickhline}{%
    \noalign {\ifnum 0=`}\fi \hrule height 1pt
    \futurelet \reserved@a \@xhline
}
\newcolumntype{"}{@{\hskip\tabcolsep\vrule width 1pt \hskip\tabcolsep}}
\makeatother
\newcolumntype{'}{@{\hskip\tabcolsep\vrule width 1pt}}
\makeatother
\newcolumntype{`}{@{\vrule width 1pt \hskip\tabcolsep}}
\makeatother

\DeclareMathOperator*{\E}{\mathrm{E}}
\let\Pr\relax
\DeclareMathOperator*{\Pr}{\mathrm{Pr}}

\newcommand{\whp}{w.h.p. }

\DeclareMathOperator{\Input}{\mathrm{X}}
\DeclareMathOperator{\Inh}{\mathrm{Z}}

\DeclareMathOperator{\NumInh}{\mathrm{\alpha}}
\DeclareMathOperator{\Output}{\mathrm{Y}}
\DeclareMathOperator{\Net}{\mathrm{N}}
\DeclareMathOperator{\weightZ}{w^{inh}}
\DeclareMathOperator{\weightY}{w^{out}}
\DeclareMathOperator{\weightS}{w^{self}}
\DeclareMathOperator{\weightX}{w^{input}}
\DeclareMathOperator{\BiasOut}{b^{out}}

\DeclareMathOperator{\ExpectedT}{\mathcal{ET}}
\DeclareMathOperator{\HighProT}{\mathcal{HT}}

\newcommand{\norm}[1]{\|#1\|}

\makeatletter

\newtheorem*{rep@theorem}{\rep@title}
\newcommand{\newreptheorem}[2]{%
\newenvironment{rep#1}[1]{%
 \def\rep@title{#2 \ref{##1}}%
 \begin{rep@theorem}}%
 {\end{rep@theorem}}}
\makeatother
\newtheorem{theorem}{Theorem}
\newreptheorem{theorem}{Theorem}

\newtheorem{corollary}[theorem]{Corollary}
\newtheorem{observation}[theorem]{Observation}
\newtheorem{lemma}[theorem]{Lemma}
\newtheorem*{lemma*}{Lemma}
\newreptheorem{lemma}{Lemma}

\newtheorem{claim}[theorem]{Claim}

\def\blackslug{\hbox{\hskip 1pt \vrule width 4pt height 8pt
    depth 1.5pt \hskip 1pt}}
\def\QED{\quad\blackslug\lower 8.5pt\null\par}

\def\dnsparagraph#1{\par\vspace{2pt}\noindent{\bf #1}.}

\title{Computational Tradeoffs in Biological Neural Networks:\\
Self-Stabilizing Winner-Take-All Networks}

\author{
Nancy Lynch \\
  \small MIT \\
  \small lynch@csail.mit.edu 
\and
 Cameron Musco \\
  \small MIT \\
  \small cnmusco@mit.edu
\and
Merav Parter\\
	\small MIT \\
	\small parter@mit.edu
}

\begin{document}

\maketitle

\begin{abstract}

We initiate a line of investigation into biological neural networks from an algorithmic perspective. We develop a simplified but biologically plausible model for distributed computation in \emph{stochastic spiking neural networks} and study tradeoffs between computation time and network complexity in this model. Our aim is to abstract real neural networks in a way that, while not capturing all interesting features, preserves high-level behavior and allows us to make biologically relevant conclusions.

In this paper, we focus on the important `winner-take-all' (WTA) problem, which is analogous to a neural leader election unit: a network consisting of $n$ input neurons and $n$ corresponding output neurons must converge to a state in which a single output corresponding to a firing input (the `winner') fires, while all other outputs remain silent.
Neural circuits for WTA rely on inhibitory neurons, which suppress the activity of competing outputs and drive the network towards a converged state with a single firing winner. We attempt to understand how the number of inhibitors used affects network convergence time.

We show that it is possible to significantly outperform naive WTA constructions through a more refined use of inhibition, solving the problem in $O(\theta)$ rounds in expectation with just $O(\log^{1/\theta} n)$ inhibitors for any $\theta$. An alternative construction gives convergence in $O(\log^{1/\theta} n)$ rounds with $O(\theta)$ inhibitors.
We compliment these upper bounds with our main technical contribution, a nearly matching lower bound for networks using $\ge \log \log n$ inhibitors. Our lower bound uses familiar indistinguishability and locality arguments from distributed computing theory applied to the neural setting. It lets us derive a number of interesting conclusions about the structure of any network solving WTA with good probability, and the use of randomness and inhibition within such a network.

\end{abstract}

\thispagestyle{empty}
\clearpage
\setcounter{page}{1}

\section{Introduction}

In this paper, 
we study biological neural networks from an algorithmic perspective, focusing on understanding 
tradeoffs between computation time and network complexity. 
We use a biologically plausible yet simplified neural computational model.
Our goal is to abstract real neural networks in a way that, while not capturing all interesting features, preserves high-level behavior and allows us to make biologically relevant conclusions.

\subsection{Model and Problem Statement}
\paragraph{Model.} We work with \emph{spiking neural networks} (SNNs) \cite{maass1996computational,maass1997networks,gerstner2002spiking,izhikevich2004model,habenschuss2013stochastic}, in which neurons fire in discrete pulses, in response to a sufficiently high membrane potential. This potential is induced by spikes from neighboring neurons, which can have either an excitatory or inhibitory effect (increasing or decreasing the potential). Our model is \emph{stochastic} -- each neuron functions as a probabilistic threshold unit, spiking with probability given by applying a sigmoid function to the membrane potential. In this respect, our networks are similar to the popular Boltzmann machine \cite{ackley1985learning}, with the important distinction that synaptic weights are not required to be symmetric and, as observed in nature, neurons are either strictly inhibitory (all outgoing edge weights are negative) or excitatory. 
While a rich literature focuses on deterministic threshold circuits \cite{minsky1969perceptrons,hopfield1986computing}
we employ a stochastic model as 
it is widely accepted that neural computation is inherently stochastic \cite{allen1994evaluation,shadlen1994noise,faisal2008noise}, and that while this can lead to a number of challenges, it also affords significant computational advantages \cite{maass2014noise}. 

\paragraph{The WTA Problem.} We focus on the Winner-Take-All (WTA) problem, which is one of the most studied problems in computational neuroscience. 
A WTA network has $n$ input neurons, $n$ corresponding outputs, and a set of auxiliary neurons that facilitate computation. 
The goal is to pick a `winning' input -- that is, the network  should produce a single firing output which corresponds to a firing input. Often the winning input is  the one with the highest firing rate, in which case WTA serves as a neural max function. We focus on the case when all inputs have the same or similar firing rates, in which case WTA serves as a leader election unit. 

WTA is widely applicable, including in circuits that implement visual attention via WTA competition between groups of neurons that process different input classes \cite{koch1987shifts,lee1999attention,itti2001computational}. It is also the foundation of competitive learning \cite{nowlan1989maximum,kaski1994winner,gupta2009hebbian}, in which classifiers compete to respond to specific input types. More broadly, WTA is known to be a powerful computational primitive \cite{maass1999neural,maass2000computational} -- a network equipped with WTA units can perform some tasks significantly more efficiently than with just  linear threshold neurons (McCulloch-Pitts neurons or perceptrons). 

\paragraph{Related Work.}
Due to its importance, there has been significant work on WTA, including in biologically plausible spiking networks \cite{lazzaro1988winner,yuille1989winner,thorpe1990spike,coultrip1992cortical,wang2003k,oster2006spiking,oster2009computation,al2015inherently}.
This work
is extremely diverse -- while mathematical analysis is typically given, different papers show different guarantees and apply varying levels of rigor. To the best of our knowledge, no asymptotic time bounds (e.g., as a function of the number of inputs $n$) for solving WTA in spiking neural networks have been established.\footnote{Aside from immediate bounds for deterministic circuits using many ($\Omega(n)$) auxiliary neurons \cite{lazzaro1988winner,maass2000computational}.} 
Additionally, previous analysis often requires a specific initial network state to show convergence and does not show that the network is self-stabilizing and converges from an arbitrary starting state, as is necessary in a biological system.

Within theoretical computer science, our work is most
inspired by: (1) work on the computational power of spiking neural networks, including the power of WTA as a black-box primitive, most notably by Maass et al. \cite{maass1997networks,maass1999neural,maass2000computational} (2) the pioneering work of Les Valiant on the
neuroidal model \cite{valiant2000circuits,valiant2000neuroidal,valiant2005memorization} and (3) self-stabilization algorithms in distributed networks  \cite{dolev2000self,lynch1996distributed}. We survey this literature in more depth in Appendix \ref{subsec:additionalrw}.

\paragraph{Basic WTA Networks.}
We restrict our attention to a simple network structure that can implement WTA efficiently using a small number of auxiliary neurons. A network consists of three layers: $n$ input neurons $\Input$, $n$ output neurons $\Output$, and $\alpha$ auxiliary neurons $\Inh$. We usually assume all auxiliary neurons are inhibitory, however in Appendix \ref{sec:excitat} give extensions to the more general case where we allow auxiliary neurons to also be excitatory.
Similar to well-known feedforward networks, all synaptic connections are between layers\footnote{Although, due to recurrent connections the network convergence time is not synonymous with the number of layers.} with the exception of an excitatory self-loop from each output $y_i$ to itself. This basic structure is biologically plausible; in particular self-loops and reciprocal excitatory-inhibitory connections (as implemented in our networks) are used in many biological models of WTA computation \cite{yuille1989winner,coultrip1992cortical,roux2015tasks}.

It is well known that inhibition is crucial for solving WTA -- outputs compete for activation via \emph{lateral inhibition} or \emph{recurrent inhibition} \cite{coultrip1992cortical,roux2015tasks}. In our network, outputs fire in response to stimulation by their corresponding inputs, thereby stimulating inhibitors which suppress the activity of other outputs. Once a single winner is selected, it must remain distinguished from the remainder of the outputs. This is achieved via positive feedback -- a consistently firing output will tend to continue firing due to its excitatory self-loop.

\subsection{Our Contribution}
\paragraph{Computational Tradeoffs.}
We explore the tradeoff between the number of inhibitors $\alpha$ used in a WTA network (i.e., the complexity of the network) 
and the time required to select a winning output (to converge to a WTA state). 
In artificial neural networks, inhibitory and excitatory connections are often treated equally, as connections with either positive or negative weights. However, in reality, neurons themselves are either inhibitory or excitatory and do not have outgoing connections of both types. There are many fewer inhibitors (around 15\% of the neural population \cite{rudy2011three,gomez2000gabaergic}), and they typically have restricted connectivity structures, often inhibiting just neurons in their local vicinity \cite{maass2000computational}. This gives natural motivation to understanding how the number of inhibitors used in a network affects its computational power.
We give two main results:
\begin{theorem}[Upper bound]
\label{thm:upper}
(1) For any $\alpha \ge 2$ there exists a basic WTA network with $\alpha$ inhibitors that, from any \emph{arbitrary} starting configuration, converges to a valid WTA state in $O(\NumInh \log^{1/\NumInh} n)$ expected time. 
(2) For any $\theta \ge 1$ there exists a basic WTA network with $\alpha = O(\theta \log^{1/\theta} n)$ inhibitors that converges in $O(\theta)$ expected time.
\end{theorem}
For $\alpha \ge \log \log n$ the above gives runtime $\tilde O\left (\frac{\log \log n}{\log \alpha}\right )$. We give a nearly matching lower bound in this case, which holds even if we allow both excitatory and inhibitor auxiliary neurons.
\begin{theorem}[Lower bound]
\label{thm:lower}
Any basic WTA network with $\alpha$ inhibitors requires $\Omega(\log\log n/\log \NumInh)$ rounds to solve WTA in expectation.
\end{theorem}
\paragraph{Upper Bound Techniques.} Our upper bounds are based on random competition between outputs that fire in response to stimulation from their firing inputs. One ``stability'' inhibitor is responsible for maintaining a WTA steady-state: as soon as just a single output fires in a round it becomes the winner of the network. Its positive feedback self-loop allows it to keep firing in subsequent rounds, while all other outputs do not fire due to inhibition from the stability inhibitor.

In order to reach a round in which just a single output fires, we employ a number of ``convergence inhibitors''. Ideally, if $k$ competing outputs fire in a round, each would fire in the next round with probability $1/k$ and we would have just a single firing output with constant probability. We can approximate this behavior using $\lfloor \log n \rfloor$ convergence inhibitors, each of which acts as a threshold circuit and fires whenever $\ge 2^i$ outputs fire for $i \in 1,...,\lfloor \log n \rfloor$. Thus when $k$ outputs fire, approximately $\log(k)$ inhibitors fire, the inhibition causes outputs to continue firing with probability $\Theta(1/k)$, and  convergence is achieved in constant rounds in expectation. This technique implicitly splits the possible number of firing outputs into $\log n$ \emph{density classes} and uses one inhibitor to ensure fast convergence from each class. To obtain more general runtime tradeoffs, we will use density classes of increasing coarseness, with the inhibitors assigned to each density classes ensuring that the number of firing outputs decreases in few rounds until it falls into a finer density class, and eventually until just a single output fires.


%

\paragraph{Lower Bound Techniques.}
Our lower bound shows that \emph{any} network which solves WTA must have a similar structure to the network described above. The inhibitory neurons can always be roughly be divided into two classes: stability and convergence inhibitors. Further, while randomness is important in breaking symmetry between competing inputs, we show that
in any efficient network, the inhibitors behave in a \emph{nearly deterministic} manner, matching behavior seen in our upper bounds. After significantly constraining inhibitor behavior, we are able to analyze how any network which solves WTA behaves on inputs with varying numbers of firing neurons. Specifically, we consider $\Theta(\log n)$ different inputs configurations, with geometrically increasing numbers of firing input neurons, ranging from $O(1)$ to $O(n)$.  
We show that, after $t$ rounds, with good probability, the network \emph{does not distinguish between} (i.e. behaves identically for) $\Theta(\log n/\alpha^t)$ inputs. 

As long as $\log n/\alpha^t > 2$, after $t$ rounds, there are at least two inputs not distinguished by the network, and so on which the network cannot achieve WTA with good probability. This yields our lower bound of $t=\Omega(\log\log n/\log \alpha)$ rounds in expectation. Our argument uses techniques familiar in distributed computing theory \cite{lynch1989hundred}, showing that limited local information prevents outputs from behaving in distinct manners for a large number of density classes in each round.

We obtain a corresponding lower bound for the number of rounds required to solve WTA \emph{with high probability} by showing that in general, the high probability runtime is $\Omega(\log n/\log \log \log n)$ times the expected runtime. This nearly matches the $O(\log n)$ gap which can be achieved by noting that in $O(\log n)$ runs, any network will converge within its expected runtime at least once with high probability. 
Our conversion result shows that, in our setting, 
expected runtime is a more natural metric -- it is controlled by the number of inhibitors used, whereas the high probability runtime is just a function of expected runtime, independent of the number of inhibitors

\begin {table}[H]\label{tab:upperBounds}
\begin{center}
\begin{tabular}{ | c | c | c | }
  \hline 
\textbf{Inhibitors} & \textbf{Lower Bound (Expected Time)} & \textbf{Upper Bound (Expected Time)}  \\
 \hline			
  Unbounded & $\Omega(1)$ ($\Omega(\log n)$ high probability time) & $O(1)$ with $\NumInh=\Theta(\log^{1/c} n)$  \\
 \hline			
  $1$ & $\Omega(n^c)$ & $O(n^c)$ \\
\hline 
  $2$ & $\Omega(\log n/\log\log n)$ & $O(\log n)$  \\
 \hline			
  $\alpha$ & $\Omega(\log\log n/ \log \NumInh)$ & 
\begin{tabular}{@{}c@{}}$O(\alpha \cdot \log^{1/\alpha} n)$, for $\NumInh = O(\log\log n)$ \\ 
$\tilde O\left (\frac{\log\log n}{\log \NumInh}\right )$, for $\NumInh = \Omega(\log\log n)$\end{tabular}\\
\hline
\end{tabular}
\caption{Expected Time vs. Number of Inhibitors Tradeoff in Basic WTA Networks.}
\end{center}
\end{table}

\subsection{Biological Insights in Our Results}
Previous work has conjectured that widespread use of simple WTA implementations in the brain may explain how complex computation is possible even when inhibition is relatively limited and localized \cite{maass2000computational}. Our work shows that WTA can be achieved and maintained efficietly using very few inhibitors and with a very simple connectivity structure.

Our upper and lower bound constructions have a common take home message that may shed some light into the biological implementations of WTA networks. For instance, the division of inhibitors into ``task preservers" (stability inhibitors) and ``task solvers" (convergence inhibitors) seems fundamental. 
%
%
Further, while randomness is crucial as it allows for symmetry breaking amongst competing outputs, it appears (both in the upper bounds and the corresponding lower bound) that in optimal networks the inhibitors behave almost as deterministic threshold circuits, firing with high probability whenever the number of firing outputs is above a certain level. 
This presents an interesting dichotomy -- while randomness is necessary computationally, it also has a cost in leading to unpredictable behavior amongst the inhibitors which `control' the network.

\paragraph{Road Map:} In Sec. \ref{sec:model} we describe our spiking neural network model and specify the WTA problem. In Sec. \ref{sec:warmup} we give two warm up examples of WTA networks to illustrate the tradeoff between convergence time and network size. The first has two inhibitors and converges to the WTA state within $O(\log n)$ rounds in expectation. The second has $O(\log n)$ inhibitors and $O(1)$ expected runtime. In Sec. \ref{sec:generalUpperBounds}, we provide more delicate constructions for any number of inhibitors $\NumInh$.  Our key technical result appears in Sec. \ref{sec:mainLower} where we provide a runtime lower bound (both for expected and high probability time) for circuits using $\alpha$ inhibitors, for any $\alpha$. Our lower bound nearly matches our upper bounds for $\alpha=\Omega(\log\log n)$. Missing proofs are deferred to the appendix.

\section{Neural Network Model}\label{sec:model}
A \emph{Spiking Neural Network} (SNN) $\Net=\langle \Input,\Output,\Inh, w,b \rangle$ consists of $n$ input neurons $\Input=\{x_1, \ldots, x_{n}\}$, $n$ output neurons $\Output=\{y_1, \ldots, y_{n}\}$, and $\NumInh$ auxiliary neurons $\Inh = \{z_1,...,z_{\NumInh} \}$. The directed, weighted synaptic connections between $\Input$, $\Output$, and $\Inh$ are described by the weight function $w: [\Input \cup \Output \cup \Inh] \times [\Input \cup \Output \cup \Inh] \rightarrow \mathbb{R}$. The in-degree of every input neuron $x_i$ is zero.
Each neuron is either inhibitory or excitatory: 
if $v$ is inhibitory $w(v,u)\leq 0$ for every $u$, and if $v$ is 
excitatory $w(v,u)\geq 0$ for every $u$. 
Finally, for any neuron $v$, $b(v) \in \mathbb{R}_{\geq 0}$ is the activation bias -- as we will see, roughly, $v$'s membrane potential must reach $b(v)$ in order for a spike to occur with good probability.

\paragraph{The Basic WTA Network and its Dynamics:} We focus on a restricted class of \emph{basic SNNs}, in which all auxiliary neurons are inhibitory, inputs only connect to their corresponding outputs, and there are no connections within the inhibitory or output layers, aside from an excitatory self-loop from each output to itself. All outputs have identical parameters, i.e. bias values and edge weights.

We introduce some more concise notation to describe basic SNNs. Let $\weightX>0$  be the  synaptic weight from each input $x_j$ to its corresponding output $y_j$. Let $\weightS >0$ be the weight of the excitatory self-loop from output $y_j$ to itself. Let $\weightZ_{j}\leq 0$ be the weight of the inhibitory synapses from inhibitor $z_j$ to each output neuron. Conversely, let $\weightY_{j}\geq 0$ be the weight of the excitatory synapses from each output in $\Output$ to inhibitor $z_j$. Finally, let $\BiasOut$ be the bias value for each output neuron.
For an illustration of the basic architecture, see Figure \ref{fig:basiccircuit}.
\begin{figure}[t]
\centering
\includegraphics[width=0.38\textwidth]{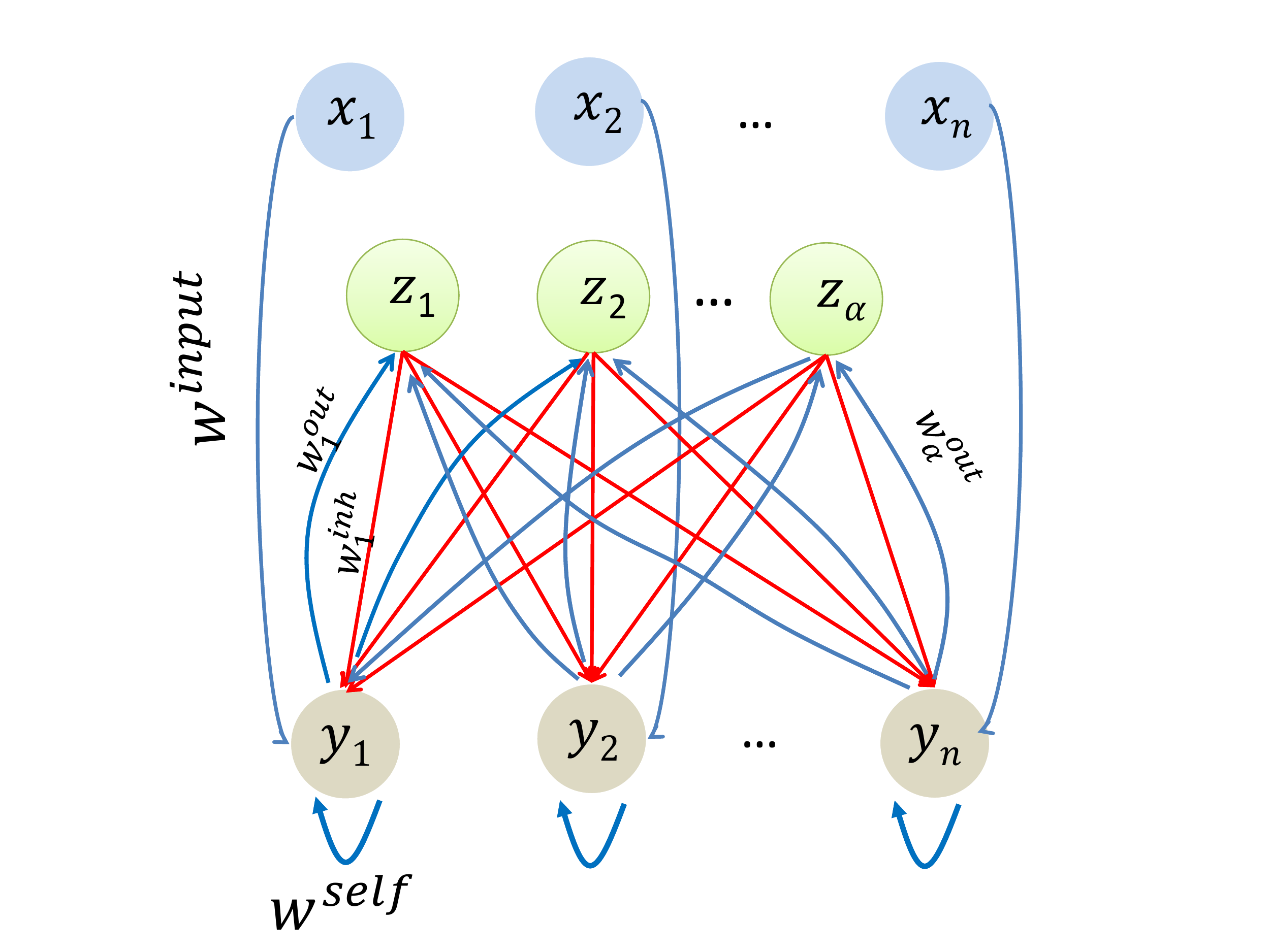}
\caption{Basic WTA Network structure.}
\label{fig:basiccircuit}
\end{figure}
%

The network evolves in discrete, synchronous \emph{rounds} as a Markov chain, with an alternating dynamic between the neurons in $\Input$, $\Output$ and $\Inh$. We give in depth biological motivation in Appx. \ref{append:bio}.

Each round $t$ consists of three sub-rounds denoted by $(t,1),(t,2)$ and $(t,3)$ where the three layers inputs, outputs and inhibitors are scheduled to fire: In the first sub-round $(t,1)$ of each round $t$, the input layer fires. We consider static inputs so each $x_i$ either fires in every round or does not fire in any round. After that, in sub-round $(t,2)$ the output neurons in $\Output$ spike with probabilities dependent on their membrane potentials. Finally, in sub-round $(t,3)$ the inhibitors in $\Inh$ spike in response to their potentials. 
The firing probability of every neuron depends on the firing status of its neighboring neurons in the preceding three sub-rounds (i.e., a length of one round).
This probabilistic firing is modeled using a standard sigmoid 
function. 
For each neuron $u$, let $u^{(t,k)}=1$ if $u$ fires (i.e., generates a spike) in sub-round $(t,k)$ for $k \in \{1,2,3\}$.

Since each neuron is always scheduled to fire in one of $(t,1)$, $(t,2)$ or $(t,3)$ depending on if it is in layer $\Input$, $\Output$, or $\Inh$, for convenience we will often omit the sub-round notation, writing 
$u^t=1$ if $u$ fires in \emph{one} of the sub-rounds $(t,k)$. We call $u^t$, the \emph{firing state} of $u$ in round $t$. Informally, we say that $u$ \emph{fires in round $t$} if $u^t = 1$.
For each output $y_j \in \Output$ , let $pot(y_j,t)$ denote the membrane potential at sub-round $(t,2)$ and $p(y_j,t)$ denote
the corresponding firing probability. These values are calculated as:
\begin{align}
\label{eq:potentialOut}
pot(y_j,t)=  (x_j^{(t,1)} \weightX) + (y_j^{(t-1,2)} \cdot \weightS) &+ \left [\sum_{z_i \in \Inh } z_i^{(t-1,3)} \cdot \weightZ_i \right ]  -\BiasOut \nonumber\\
&\text{ and }p(y_j,t)=\frac{1}{1+e^{-pot(y_j,t)/\lambda}}
\end{align}
where $\lambda > 0$ is a \emph{temperature parameter}, which determines the steepness of the sigmoid. Note that \eqref{eq:potentialOut} incorporates excitatory and inhibitory effects from any spikes occurring within the three sub-rounds before the outputs spike in sub-round $(t,2)$. Specifically, this includes input spikes in sub-round $(t,1)$ along with output and inhibitory spikes in sub-rounds $(t-1,2),(t-1,3)$ respectively.
Applying the same rules, in sub-round $(t,3)$, each inhibitor in $\Inh$ fires with probability $p(z_j,t)$ calculated as:
\begin{align}
\label{eq:potentialInh}
pot(z_j,t)=  \left [\sum_{y_i \in \Output } y_i^{(t,2)} \cdot \weightY_j \right ] - b(z_j) \text{ and }p(z_j,t)=\frac{1}{1+e^{-pot(z_j,t)/\lambda}}.
\end{align}
Again \eqref{eq:potentialInh} incorporates effects from relevant spikes within three sub-rounds $(t-1,3), (t,1)$ and $(t,2)$. However, since the inhibitors are connected only to the outputs, the only sub-round that affects them is $(t,2)$. 
After the inhibitors fire, computation proceeds to round $t+1$, beginning with the firing of the inputs.
\def\APPENDBIO{
In our network, the timing of the neural spikes is determined by two biological parameters, namely, the \emph{refractory period} $\beta$ and the \emph{response latency}, $\Delta$.  The refectory period is the time during which stimulus given to the neuron would not cause a second action potential. The response latency is the delay between the time the action potential reaches the presynaptic terminal of the input neurons and the time the postsynaptic output neuron sends out an action potential (assuming it does).
In our setting we consider the case where $\Delta < \beta$ since for connected neurons in close proximity to each other, and inhibitory neurons with primarily local connections, the response delay is a few hundred of micro-seconds whereas the refractory time is several milliseconds \cite{sabatini1999timing}. WTA networks are basic, local neural primitives that are not believed to involve long range connections, justifying our assumption. 

Every round corresponds to an interval between two pulses of the inputs (hence a round lasts $\beta$ milliseconds). At the beginning of every round, the input layer spikes (at sub-round $(t,1)$ in the notation of our discrete model).
The spikes generated by the inputs invoke an alternating dynamic between the three layers in the circuit. Specifically, with a delay of $\delta$ milliseconds after the input's spike, the outputs spike with probability that is proportional to their total synaptic strengths (in sub-round $(t,2)$). As shown in equation \eqref{eq:potentialOut}, this potential incorporates any spikes which occurred within a $\beta$ millisecond preceding window -- the input spikes in sub-round $(t,1)$ ($\Delta$ milliseconds before), the inhibitor spikes in sub-round $(t-1,3)$ ($\beta -\Delta$ milliseconds before), and the neuron's own self-excitatory output spike in sub-round $(t-1,2)$, $\beta$ milliseconds before. $\Delta$ milliseconds after the outputs spike, the inhibitors spike in sub-round $(t,3)$, again incorporating spikes that occurred with a $\beta$ millisecond window, which due to their limited connectivity structure, just includes the spikes of $\Output$ in sub-round $(t,2)$.

}

\paragraph{Temperature and Background Noise.}
It is clear that the temperature $\lambda$ does not affect the computational power of the network as we can simply adjust all synapse weights and neuron biases by a factor of $\lambda/\lambda'$ to simulate a network with temperature $\lambda'$. Hence, we always choose a $\lambda$ that makes exposition easier.
We assume that neurons in $\Inh,\Output$ have bias $b(v)=\Omega(\lambda \log n)$, so they do not fire with probability $1-1/(1+e^{-c \cdot \log n})=1-1/n^c$ when they receive no external stimulation. We call this the \emph{no-background noise} assumption: the network is quiet when no input is introduced.  

\paragraph{System Configuration.} The configuration $\mathcal{C}^t=(\Input^t,\Output^t,\Inh^t)$ in round $t$ is defined by the firing states\footnote{The firing state of a neuron is a binary number indicating if it is firing or not.} of the corresponding neurons in round $t$ where $\Input^t = [x_1^t,...,x_n^t]$ and $\Output^t$ and $\Inh^t$ are defined analogously. Recall that $x_i^t=1, y_i^t=1, z_i^t=1$ if the input $x_i$ (output $y_i$, inhibitor $z_i$) fires in sub-round $(t,1)$ (resp., $(t,2), (t,3)$).  
We consider a static input setting where $\Input^t =\Input $ for all $t$.\footnote{Note however that our model can easily handle non-static inputs. All algorithms given will converge from an arbitrary initial configuration and so will converge if $\Input$ changes.} We abuse notation slightly, thinking of $\Input$ as a vector of binary input values where 
$x_j=1$ indicates that $x_j$ fires in every round ($x_j^t = 1$ for all $t$) and $x_j=0$ implies that $x_j$ never fires ($x_j^t = 0$ for all $t$). 
In the initial configuration $\mathcal{C}^0$, $\Input^0 = \Input$, $\Output^0$ can be arbitrary, and $\Inh^0$ is determined as in any round according to equation \eqref{eq:potentialInh}.

\paragraph{The WTA Problem.}
A binary winner-take-all network given $n$ inputs should converge to having a single firing output corresponding to a firing input (the `winner'), if one exists. Formally, given $\Input \in \{0,1\}^n$, let 
$f(\Input) = \{\Output \in \{0,1\}^n ~| ~y_i \le x_i \text{ } \forall i\text{ and } \norm{\Output}_1 = \min(1, \norm{\Input}_1) \}$
where $\norm{\cdot }_1$ is the standard $1$-norm, used to denote the number of firing neurons in a set.

We say $\Net$ \emph{satisfies WTA} in round $t$ if $\Output^t \in f(\Input)$. We say $\Net$ \emph{converges to WTA} in $t$ rounds with probability $1-\delta$ if
for every input $\Input \in \{0,1\}^n$ and every initial output configuration $\Output^0$, with probability at least $1-\delta$, 
$\Output^t \in f(x)$ and $\Output^{t'} = \Output^t$ for all $t' \in [t+1,t+n^{c}]$ where $c$ is a positive constant.
That is,
the network satisfies WTA in round $t$ and maintains the satisfying configuration for polynomial in $n$ subsequent rounds.
As our neurons are inherently probabilistic, our definition of convergence is as well -- we will never be able to avoid occasional random deviations from a correct output state and so just demand that the state is maintained a large number of rounds.


We let $\ExpectedT(\Net)$ denote the maximum expected time required to converge to WTA, taken over all possible inputs $\Input$ and initial output configurations $\Output^0$. 
In the same manner, $\HighProT(\Net)$ denotes the maximum time required for the network to converge to WTA with high probability.\footnote{Throughout, \emph{with high probability} (w.h.p.) refers to events occuring with probability $\ge 1- 1/n^c$ for constant $c$.}

  

\section{Warm Up: Two Simple Networks for WTA}
\label{sec:warmup}
We begin by presenting two WTA networks that represent two extremes of the inhibitor-time tradeoff. They also illustrate the rough intuition that will appear in our later network constructions and lower bound strategies.

\paragraph{WTA with two inhibitors.}
In our two inhibitor network we have $\Inh  = \{z_s, z_c\}$. The neuron $z_s$ is a \emph{stability} inhibitor that maintains the WTA state once it has been reached. It fires \whp in sub-round $(t,3)$ whenever at least one output fires in sub-round $(t,2)$. The neuron $z_c$ is a \emph{convergence} inhibitor that fires \whp whenever WTA has not yet been reached -- i.e. whenever $\ge 2$ outputs fire in sub-round $(t,2)$.

We set the weights connecting $z_s$ and $z_c$ to the outputs such that when both fire in round $t$, any output that fired in round $t$ will fire with probability $1/2$ 
in round $t+1$. Any output that \emph{did not fire} in round $t$ will not fire in round $t+1$ \whp as it will not have an active excitatory self-loop and so its membrane potential will be too low to overcome the inhibition.

In this way, as long as $\ge 2$ outputs fire in round $t$, both inhibitors fire \whp and the high level of inhibition causes outputs to `drop out of contention' for the winning position with probability $1/2$. After $O(\log n)$ rounds, nearly all the outputs stop firing and with constant probability there is a round in which exactly $1$ output fires.
Once this round occurs, $z_c$ ceases firing \whp and just $z_s$ fires. This decreased level of inhibition allows the winner to keep firing, as it is offset by the winner's excitatory self-loop. However, it prevents any other output, whose excitatory self-loop is inactive, from firing \whp 
See Figure \ref{fig:two} in Appendix \ref{append:twoi} for illustration of the network with its edge weights.
%
%
We analyze the network in depth in \ref{append:twoi}, showing convergence given any input $\Input$ and initial output configuration $\Output^0$, and yielding:

\begin{theorem}
\label{thm:twoihb}
There exists a basic WTA network $\Net$ with $\NumInh = 2$ inhibitors and $\mathcal{ET}(\Net) = O(\log n)$ and $\mathcal{HT}(\Net) = O(\log^2 n)$.
\end{theorem}
\def\APPENDTWO{
\begin{figure}[H]
\centering
\includegraphics[width=0.4\textwidth]{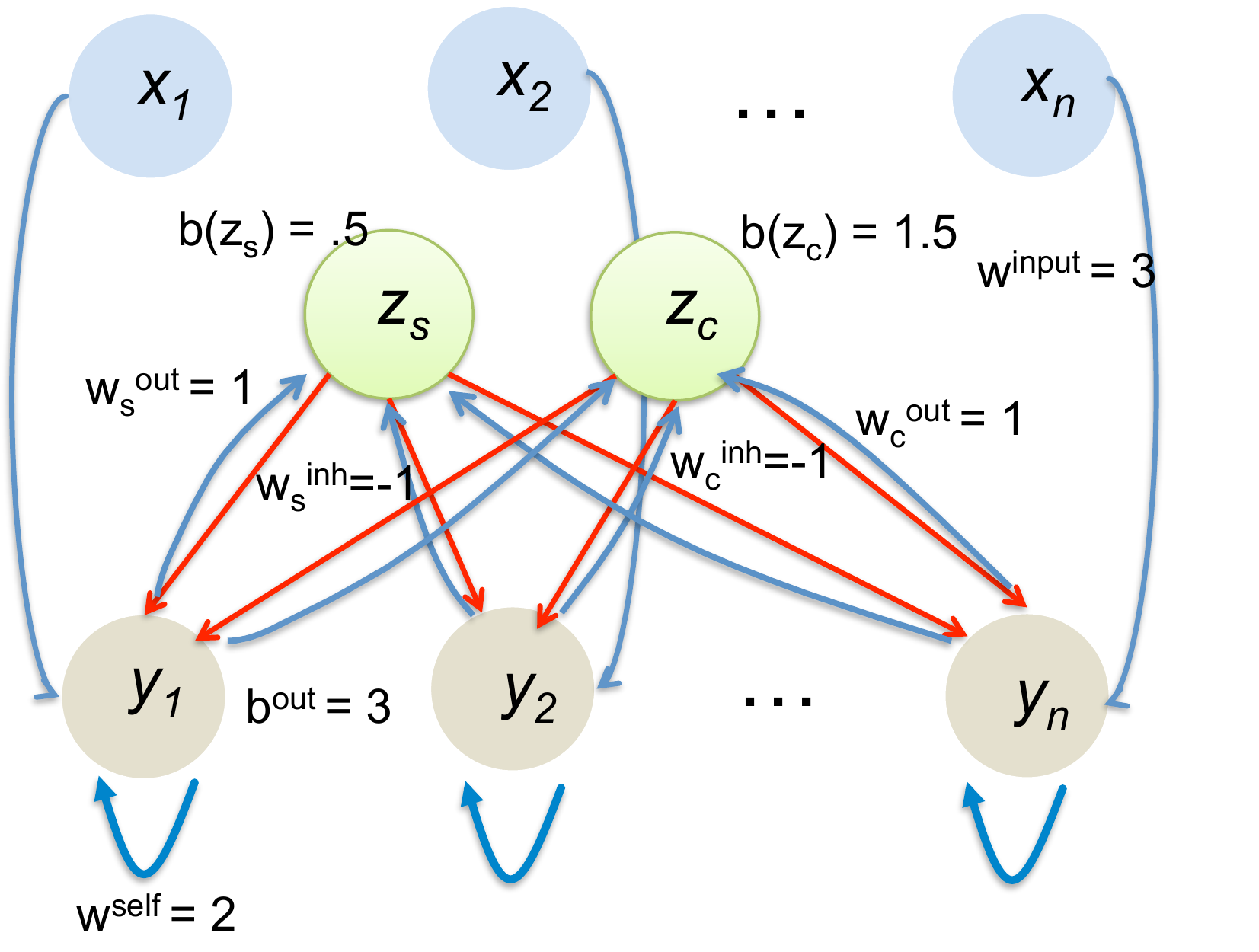}
\caption{Two Inhibitor WTA Network}
\label{fig:two}
\end{figure}
\dnsparagraph{Proof of Theorem \ref{thm:twoihb} (Two Inhibitor Upper Bound)}

Formally the parameters of the network are set as follows: assume w.l.o.g. that $\lambda = 1/(c_1\log n)$ for large constant $c_1$.
For both inhibitors, set the excitatory output to inhibitor weights to $\weightY_s = \weightY_\ell = 1$ and $b(z_s) = .5$, $b(z_c) = 1.5$. Thus, by equation \eqref{eq:potentialInh} $z_s$ fires \whp in sub-round $(t,3)$ whenever at least one output fires in sub-round $(t,2)$, and $z_c$ fires \whp whenever at least two outputs fire.

Set the inhibitor to output weights to $\weightZ_s  = \weightZ_\ell = -1$, the excitatory input to output connection weight to $\weightX = 3$, and the excitatory output to output self-loop to $\weightS = 2$. Finally, set the output bias to $\BiasOut = 3$.

The above parameters insure that only outputs corresponding to firing inputs ever fire \whp Additionally, if we have not yet reached WTA and both $z_s$ and $z_c$ fire in sub-round $(t,3)$, any output that fired in sub-round $(t,2)$ will fire with probability $1/2$ in sub-round $(t+1,2)$. If we have reached WTA and just $z_s$ fires, any output (the winner) that fired in round $t$ will fire in round $t+1$ \whp In either case, any output that did not fire in round $t$ will not fire \whp in round $t+1$.

We now give a formal proof of the theorem.
First note that if the input $\Input = \vec{0}$ then in every round, each output has potential $pot(y_j, t) \le \weightS - \BiasOut = -1$ and so, recalling that $\lambda = 1/(c_1 \log n)$, fires with probability at most $\frac{1}{1 + e^{c_1 \log n}} \le 1/n^c$ for some large constant $c$ in any round. So \whp no outputs fire in each round, which is the valid output given $X = \vec{0}$ and so $\Net$ trivially converges to WTA. 
So for the remainder of the section we focus on the case in which $\Input$ has at least one firing input.
We show that $\Net$ satisfies the following conditions, which imply Theorem \ref{thm:twoihb}:
\begin{claim}[Stability]
\label{cl:stability}
If $\Net$ satisfies WTA in round $t$ with $y_j^t = 1$, then $\Net$ satisfies WTA in round $t+1$ with $y_j^{t+1} = 1$ \whp
\end{claim}

\begin{claim}[Convergence]
\label{lem:live}
Letting $t=c_2 \log n$ for constant $c_2$, for any input $\Input$ with $\norm{X}_1 \ge 1$ and any starting configuration $C^0$, $\Net$ satisfies WTA in round $C^{t'}$ for some $t' < t$, with constant probability.
\end{claim}
Since Claim \ref{lem:live} holds for any starting configuration, we can simply
apply it $\Theta(\log n)$ times to show that \whp within $\Theta(\log^2 n)$ rounds, there will be a round in which WTA is satisfied, and hence $\Net$ will converge to WTA by Claim \ref{cl:stability}. Additionally, it gives $\mathcal{ET}(\Net) = O(\log n)$ as letting $c_1$ be the constant probability of reaching WTA in $O(\log n)$ rounds, we have:
\begin{align*}
\mathcal{ET}(\Net) = O \left ( \sum_{i=0}^\infty (1-c_1)^i \cdot c_1\log n \right ) = O(\log n).
\end{align*}
This gives us Theorem \ref{thm:twoihb}.

\begin{proof}[Proof of Claim \ref{cl:stability}]
$\Net$ satisfies WTA in round $t$ with output $y_j$ firing, so we have 
\begin{align*}
pot(z_s,t) = 1 \cdot \weightY_s - b(z_s) = .5 \text{ and } pot(z_c,t) = 1\cdot \weightY_s - b(z_c) = -.5.
\end{align*}
Thus, recalling that we have $\lambda = 1/(c_1 \log n)$, in round $t$ $z_s$ fires with probability $\frac{1}{1 + e^{-.5c_1 \log n}} \ge 1 - 1/n^c$ for large $c$ and $z_c$ fires with probability $\frac{1}{1 + e^{.5c_1 \log n}} \le 1/n^c$ for large $c$. So \whp just $z_s$ fires in round $t$.
This gives that \whp
$$pot(y_j,t+1) = (1 \cdot \weightZ_s) + (0 \cdot \weightZ_\ell) + (1 \cdot \weightS) + \weightX - \BiasOut = -1 + 2 + 3 - 3 = 1.$$
So $y_j$ fires with probability $\frac{1}{1+e^{c_1\log n}} \ge 1-1/n^c$ in round $t+1$. In contrast, for any $j' \neq j$, $y_{j'}$ does not fire in round $t$ so we have \whp
$$pot(y_{j'},t+1) \le (1 \cdot \weightZ_s) + (0 \cdot \weightZ_\ell) + (0 \cdot \weightS) + \weightX - \BiasOut = -1 + 3 - 3 = -1.$$
Therefore $y_j'$ fires with probability $\le 1/n^c$ in round $t+1$ so WTA is satisfied with output $y_j$ firing in round $t+1$ \whp
\end{proof}

\begin{proof}[Proof of Claim \ref{lem:live}]
Recall that we only consider $\norm{X}_1 \ge 1$ as convergence to WTA is trivial when $X = \vec{0}$.
We analyze three simple cases depending the initial configuration $C^0$:

\paragraph{Case 0: No output $y_j$ with $x_j = 1$ fires in $\Output^0$.}
We first consider the subcase that no output (regardless of the value of $x_j$) fires in $\Output^0$.
In this case, $pot(z_s,0) = -b(z_s) = -.5$ and  $pot(z_c,0) = -b(z_c) = -1$ so neither inhibitor fires \whp in round $0$. So \whp all outputs with firing inputs have $pot(y_j,1) \ge \weightX - \BiasOut = 0$ and so fire with probability $\ge 1/2$ in round $1$. Since, $\Input \neq \vec{0}$, with constant probability at least one of these outputs fires in round $1$, in which case we appeal to Cases 1 and 2 below (where we re-label $C^2$ as the initial configuration $C^0$.).

Next consider the case when at least one output fires in $Y^0$, but all firing outputs correspond to non-firing inputs. In this case, we have $pot(z_s,0) \ge 1 \cdot \weightY - b(z_s) \ge .5$ and so $z_s$ fires \whp in round $0$. As noted, in any round, any output $y_j$ with $x_j = 0$ has $pot(y_j, t) \le \weightS - \BiasOut = -1$ and so does not fire \whp Additionally, since every output with $x_j = 1$ has $y_j^0 = 1$, these outputs have $pot(y_j, 1) = \weightZ_s + \weightX - \BiasOut = -1 +3 -3 = -1$ and so do not fire \whp in round $1$. So \whp in round $1$ no outputs fire and we are in the first case above.

\paragraph{Case 1: Exactly one output $y_j$ with $x_j = 1$ fires in $\Output^0$.}

By Claim \ref{cl:stability} and the fact that outputs with $x_j = 0$ do fire \whp in any round, $\Net$ satisfies WTA in round $1$ and so immediately converges to WTA. 

%
%

\paragraph{Case 2: More than one output $y_j$ with $x_j = 1$ fires in $\Output^0$.}

Let $k_t$ be the number of \emph{active} outputs in round $t$ -- that is outputs corresponding to firing inputs that fire in round $t$. 
For any round with $k_t \ge 2$, we have 
 $pot(z_s,t) \ge 2\weightY-b(z_s) = 1.5$ and  $pot(z_c,t) \ge 2\weightY-b(z_c) = 1$. So both inhibitors fire in round $t$ \whp Conditioning on this event, all active outputs have:
$$pot(y_j,t+1) = (1 \cdot \weightZ_s) + (1 \cdot \weightZ_\ell) + (1 \cdot \weightS) + \weightX - \BiasOut = -1 -1 + 2 + 3 - 3 = 0$$
and so fire with probability $1/2$ in round $t+1$. 
All inactive outputs, which did not fire in round $t$, do not have an active self loop and hence have $pot(y_j,t) = -2$ and don't fire in round $t+1$ \whp (as discussed, all outputs with $x_j = 0$ also do not fire \whp)

Conditioning on this event, 
 with probability $1/2$, $k_{t+1} \le k_t/2$.
Further, 
$$\Pr[k_{t+1} = 0] = 1/2^{k_t} \text{ and }\Pr[k_{t+1} = 1] = k_t \cdot (1/2^{k_t}) \ge \Pr[k_{t+1} = 0].$$
 So the probability of reaching $k_{t+1}=1$ and hence $\Net$ converging to WTA is at least as high as the probability of overshooting WTA and having no outputs firing in round $t+1$.

Overall, conditioning on the fact that $z_s$ and $z_c$ fire in every round in which $k_t \ge 2$ and that no output which was inactive in round $t$ fires in round $t+1$, whenever $k_t \ge 2$ it decreases by a factor of $1/2$ in round $t+1$ with good probability. So \whp within $O(\log (k_0)) = O(\log n)$ rounds there will be a round $t$ with either $k_t = 1$ or $k_t = 0$. $k_t = 1$ is at least as likely as $k_t = 0$ so with constant probability, $\Net$ converges to WTA within $O(\log n)$ rounds.
\end{proof}
}
In Appendix \ref{append:twoi}, we show that the network is optimal up to a $ \log\log n$ factor and 
\def\APPENDLOWEBOUNDTWO{
\dnsparagraph{Two Inhibitor Lower Bound}
\begin{theorem}
\label{thm:lbtwo}
For any basic WTA network $\Net$ with $\NumInh = 2$ inhibitors, $\mathcal{ET}(\Net) = \Omega(\log n/\log \log n)$ and $\mathcal{HT}(\Net) = \Omega(\log^2n/\log \log^2 n)$.
\end{theorem}
The key idea is that
the use of a stability inhibitor $z_s$ and a convergence inhibitor $z_c$ in the algorithm is not just a design choice, but is required for \emph{any} near-optimal two inhibitor WTA network.
\begin{claim}
\label{cl:twoinh}
For any basic WTA network $\Net$ with $\alpha = 2$ inhibitors and $\mathcal{ET}(\Net) = O(\log^3 n)$, one inhibitor $z_s$ fires \whp in sub-round $(t,3)$ if at least one output fires in sub-round $(t,2)$. The second inhibitor $z_c$, does not fire \whp in sub-round $(t,3)$ if just a single output fires in $(t,2)$. 
\end{claim}
\begin{proof}
Assume for contradiction that both inhibitors fire with probability $\omega(1/n^c)$ in sub-round $(t,3)$ after just a single output fires in sub-round $(t,2)$. Then, after a round $t$ in which $z_s^t = z_c^t = 1$, any output $y_j$ with $x_j = 1$ and $y_j^t = 1$ must fire \whp in round $t+1$. This is because once $\Net$ converges to WTA, when the single winning output fires in sub-round $(t,2)$, by our assumption, with relatively high $\omega(1/n^3)$ probability, both $z_s$ and $z_c$ fire in sub-round $(t,3)$. Even if this event occurs, the winning output must fire \whp in round $t+1$ to maintain WTA \whp

However, if we let $\Input = \vec{1}$ and $\Output^0 = \vec{1}$, then for some constant $c_1$, all outputs will continue firing for $\omega(n^{c_1})$ rounds \whp even if both $z_s$ and $z_c$ fire in every round. This contradicts our assumed $O(\log^3 n)$ runtime. Hence we have that at least one of the inhibitors, which we label $z_c$, fires with probability $O(1/n^c)$ in sub-round $(t,3)$ if just a single output fires in sub-round $(t,2)$.

Similarly, assume for contradiction that $z_s$ \emph{does not fire} with probability $\omega(1/n^c)$ in sub-round $(t,3)$ if a single output fires in sub-round $(t,2)$. Then, it must be that even if neither inhibitor fires in sub-round $(t,3)$, any output $y_j$ that did not fire in sub-round $(t,2)$ (i.e. $y^t = 0$), must also not fire \whp in sub-round $(t+1,2)$.
This is because, by our assumption, after WTA is reached, with probability $(1-O(n^c)) \cdot \omega(1/n^c) = \omega(1/n^c)$ neither inhibitor will fire in sub-round $(t,3)$ when just the single winning output fires in sub-round $(t,2)$. Still, all non-winning outputs must continue not firing in round $t+1$ to maintain WTA \whp

However, if we let $\Input = \vec{1}$ and $\Output^0 = \vec{0}$, since even when neither inhibitor fires in round $t$, each output does not fire in round $t+1$ \whp if it did not fire in round $t$,
it will take $\omega(n^{c_1})$ rounds (for some constant $c_1$) before even a single output fires \whp contradicting our assumed $O(\log^3 n)$ runtime.

\end{proof}

The above claim allows us to strongly
constrain the behavior of the network based on the action of the inhibitors $z_s$ and $z_c$. 
Let $p_0$ be the probability that an output $y_j$ fires in round $t+1$ given that $y^{t} = 0$, $x^t = 1$ and $z_s^t = z_c^t = 0$.
\begin{claim}\label{clm:p0}
For any basic WTA network $\Net$ with $\alpha = 2$ inhibitors and $\mathcal{ET}(\Net) = o(\log^2 n)$, $p_0 = \omega(1/\log^2 n)$.
\end{claim}
\begin{proof}
Consider $\Input$ with just two firing inputs $x_1=1$ and $x_2=1$. For any round $t$ in which $y_1^t = y_2^t = 0$, the probability that $y_1$ or $y_2$ fires in round $t+1$ is \emph{at most} $p_0$ -- since the firing of $z_s$ or $z_c$ can only decrease the probability of the outputs firing. Assuming by way of contradiction that a  $p_0 \le c_1 /\log^2n$ for some constant $c_1$, starting from $\Output^0 = \vec{0}$, with constant probability, neither output will fire for $\Omega(\log^2 n)$ consecutive rounds, and so $\Net$ cannot converge to WTA in expected $o(\log ^2 n)$ rounds.
\end{proof}
%

Let $p_{out}$ be the probability that output $y_j$ fires in round $t+1$ given $y^{t} = 1$, $x^t = 1$ and $z_s^t = z_c^t = 1$.
\begin{claim}\label{clm:pout}
For any basic WTA network $\Net$ with $\alpha = 2$ inhibitors and $\mathcal{ET}(\Net) = o(\log^2 n)$, $p_{out} = \omega(1/\log^6 n)$.
\end{claim}
\begin{proof}
Consider $\Input$ with $\Theta(\log^4 n)$ firing inputs and initial configuration $\Output^0$ where $y_j = 1$ for all $j$ with $x_j = 1$. Consider some round $t$ in which at least two outputs (corresponding to firing inputs) have fired in all rounds $t' \le t$. If either (or both) of $z_s$ or $z_c$ do not fire in round $t$, then since they face at most as much inhibition as when the network has converged to WTA, all outputs with firing inputs that fired in round $t$ fire \whp in round $t+1$. However, if both $z_s$ and $z_c$ do fire in round $t$, if $p_{out} = O(1/\log^6 n)$ then with probability $\le (1 - p_{out})^{\Theta(\log^4 n)} = 1- \Theta(1/\log^2 n)$ \emph{no} output corresponding to a firing input fires in round $t+1$. Since by Claim \ref{cl:twoinh} a single inhibitor firing is enough to maintain convergence to WTA, once these outputs do not fire in some round $t$, they do not fire again \whp until a round in which neither $z_s$ or $z_c$ fire. Then by Claim \ref{clm:p0} and a Chernoff bound (Theorem \ref{thm:simplecor}) $\omega(\log^2 n)$ of them fire \whp

So overall, we alternate between having many (between $\omega(\log^2 n)$ and $\Theta(\log^4 n)$) outputs corresponding to firing inputs and $0$ outputs with firing inputs. Each time we have many firing outputs, with probability at least $1-\Theta(\log^2 n)$ we have no firing outputs in the next round. So it takes at least $\Omega(\log^2 n)$ rounds before we have a round with exactly one valid firing output with constant probability, contradicting our assumed runtime of $\mathcal{ET}(\Net) = o(\log^2 n)$.
\end{proof}

With the above claims in place, we are ready to prove Theorem \ref{thm:lbtwo}.
Consider $\Input = \vec{1}$ and initial configuration with $\Output^0 = \vec{1}$. Let $k_t = \norm{Y^t}_1$ be the number of outputs that fire in round $t$.


Now, if $y_j$ fires in round $t$, then it fires with probability \emph{at least} $p_{out}$ in round $t+1$, since $p_{out}$ is the firing probability with maximum inhibition. Let $d = c_1\log n/p_{out}$ for some constant $c_1$. By Claim \ref{clm:pout}, $d = O(\log^7 n)$ and since $p_{out} \le 1$, trivially $d = \Omega(\log n)$
.
Starting from $\Output^0$ with all outputs firing, for $t =c_2 \frac{\log (d/n)}{\log p_{out} }$ for sufficiently small $c_2$ we have that any output fires in all rounds up to $t$ with probability 
$\theta \left ( p_{out}^t \right ) = \omega \left (\frac{d}{n} \right)$. So by a Chernoff bound (Theorem \ref{thm:simplecor}) \whp $\omega(d)$ outputs fire in all rounds $t' \le t$.

Let $t_f$ represent the first round in which $\le d$ outputs fire.  By our argument above, \whp 
\begin{align}\label{eq:tf}
t_f = \Theta(\log (n/d)/\log(1/p_{out})) = \Theta(\log n/\log(1/p_{out})) = \Omega(\log n/\log \log n)
\end{align}
by Claim \ref{clm:pout}. This gives us $\mathcal{ET}(\Net) = \Omega(\log n/\log \log n)$. So it just remains to show our lower bound on $\mathcal{HT}(\Net)$.

Since $> d$ outputs fire in round $t_f -1$, again by a Chernoff bound, \whp $k_{t_f} \ge d \cdot p_{out} = \Omega(\log n)$.
Consider any round $t > t_f$ in which $k_{t'} > 1$ for all $t' \le t$. If either of $z_s$ of $z_c$ do not fire in round $t$, then $k_{t+1} = k_t > 1$ \whp Otherwise, $\Pr [k_{t+1} = 1] = k_t \cdot p_{out} (1-p_{out})^{k_t-1}$ and:
$$\Pr [k_{t+1} = 0] = (1-p_{out})^{k_t} = \Pr [k_{t+1} = 1] ] \cdot \frac{1-p_{out}}{k_t p_{out}} \ge \Pr [k_{t+1} = 1] ] \cdot \frac{1}{\log^8 n}$$
where we use the fact that $k_t \le d = O(\log^7 n)$ and $1-p_{out} \ge \log n$ or else by \eqref{eq:tf} we would already not reach WTA \whp in $O(\log^2 n)$ rounds.

So, the probability that $k_{t+1} = 0$ is high (within a polylog n) factor of the probability that $k_{t+1} = 1$. 
So, with probability at least $\Omega(1/\log^8 n)$, $t_f$ is followed by a reset round in $0$ outputs fire before a round in which a single output fires. Further, once such a reset round occurs, then no output will fire until $z_s$ and $z_c$ don't fire in a round (and hence inhibition is lower than it is after convergence to WTA) in which case by Claim \ref{clm:p0} $\omega(n/\log^2 n)$ outputs will fire. So  \whp there will be $\Omega(\log n/\log \log n)$ rounds before another round in which $\le 1$ outputs fire.

Overall, in order to have a round in which exactly $1$  output fires \whp requires $\Omega(\log n/\log (\log^8 n)) = \Omega(\log n/\log \log n)$ resets, each taking $\Omega(\log n/\log \log n)$ rounds, and giving our final lower bound of $\Omega(\log^2 n/\log \log^2 n)$.
}
in Appendix \ref{appenx:oneinhib} we show that it represents a critical point in the inhibitor-time tradeoff: any network with just one inhibitor requires $\Omega(n^c)$ rounds to solve WTA. Essentially, it is not possible for a single inhibitor to implement the two opposing tasks of stability and convergence. 

\paragraph{WTA with $O(\log n)$ inhibitors.} Our second network represents another
extreme point of the inhibitor-time tradeoff, using $\alpha=O( \log n ) $ inhibitors to achieve $O(1)$ expected convergence time. 

The idea is to approximate the ideal behavior in which outputs fire with probability $1/k_t$ in round $t+1$ if $k_t$ outputs fired in round $t$.
As in our two inhibitor algorithm, we have a single stability inhibitor $z_s$ that fires \whp whenever at least one output fires and insures that as soon as a single output
fires in a round, the network converges to WTA.
%
%
We then have $\lceil \log n \rceil  - 1$ convergence inhibitors $z_1,...z_{\alpha-1}$.
We set the bias of the $z_i$ to $b(z_i) = 2^{i} - .5$ and set $\weightY_i = 1$ for all $i$. In this way, $z_i$ fires \whp in round $t$ whenever $\ge 2^{i}$ outputs fire.
We set the inhibitor to output weights to $\weightZ_i = \Theta(\lambda)$ for all $i$. Thus, when $k_t \in [2^{i}, 2^{i+1})$, \whp inhibitors $z_1,...,z_{i}$ all fire (while $z_{i+1},...,z_{\alpha-1}$ do not).
The total inhibition from the inhibitors is thus $\Theta(i \lambda)$ and hence each of the $k_t$ outputs fire with probability $1/(1+e^{\Theta(i)})\approx 1/2^i \approx 1/k_t$ in round $t+1$.
In expectation (and with constant probability) there will be exactly one firing output, giving an expected runtime of just $O(1)$ rounds to reach WTA. 
In Appendix \ref{appenx:logn}, we give a full analysis, yielding: 
\begin{theorem}
\label{theorem:logn}
There exists a basic WTA network $\Net$ with $\alpha = O(\log n)$ inhibitors, $\mathcal{ET}(\Net) = O(1)$ and $\mathcal{HT}(\Net) = O(\log n)$.  
\end{theorem}
\def\APPENDLOGNINH{
}

Vacuously, no network can beat this expected runtime. We also show in Appendix \ref{appenx:logn} that no network can do better with high probability: even with an unlimited number of inhibitors, $\Theta(\log n)$ rounds are requires to solve WTA \whp
Intuitively, as long as WTA has not yet been reached in round $t$, there is no single distinguished output. All outputs have identical connections to $\Input,\Inh$ so each active output fires with the \emph{same} probability $p$ in round $t+1$. 
Hence the probability that a single output becomes distinguished (is the only one to fire) is $k_t \cdot p(1-p)^{k_t-1}$, which is bounded by a constant for all $k_t, p$. Thus, converging to the WTA state \whp takes at least $\Omega(\log n)$ rounds.

\section{WTA with $\alpha\geq 2$ Inhibitors}
\label{sec:manyin}

The above results give a rough outline of the tradeoff between the number of inhibitors used and the achievable runtime for WTA. We now explore this tradeoff in more depth for general $\alpha \in (2,\log n]$ 

\subsection{Upper Bound Networks}\label{sec:generalUpperBounds}
We first show that both our two inhibitor and $\lceil \log n \rceil$ inhibitor networks can be improved significantly with modest increases in the number of inhibitors or runtime used.
We can (up to constant factors) match the runtime of the $\lceil \log n \rceil$ inhibitor network with just $O(\log^{1/c} n)$ inhibitors for any $c$. Additionally, for any $\alpha \ge \log \log n$ we can achieve expected runtime $O \left (\frac{\log \log n \log \log \log n}{\log \alpha}\right)$, nearly matching our main lower bound of Section \ref{sec:mainLower}.

\begin{theorem}
\label{thm:upper_symmetricalpha}
For any integer $\theta$, there is a basic WTA network $\Net$ with $\alpha = O(\theta \log^{1/\theta} n)$ inhibitors, $\mathcal{ET}(\Net) = O \left (\theta \right )$, and $\mathcal{HT}(\Net) = O \left (\theta \log n \right )$.
\end{theorem}

For $\alpha \ge \log \log n$, writing $\alpha = \log \log^x n$ for $x \ge 1$ if we set $\theta = \frac{c_1\log \log n \log \log \log n}{\log \alpha} = \frac{c_1\log \log n}{x}$ then the number of inhibitors required is: $\frac{c_1\log \log n}{x} \cdot e^{x/c_1} \le \log\log^x n \le \alpha$ for small enough $c_1$.

\begin{proof}[Proof Sketch]
To see the high level idea, consider the case of $\theta=2$. 
We will $2\sqrt{\log n}$ inhibitors which are divided into two classes: $\sqrt{\log n}$ coarse inhibitors and $\sqrt{\log n}$ fine inhibitors. The edges from the fine inhibitors to outputs have weight $-1$ and the edges from coarse inhibitors to outputs have weight $-\sqrt{\log n}$. All the edges from the outputs to the inhibitors have weight $1$. We set the bias values of the inhibitors such that: (1) the $i^{th}$ coarse inhibitor fires if the number of active outputs is at least $2^{i\sqrt{\log n}}$ and (2) the $i^{th}$ fine inhibitor fires if the number of active outputs is at least $2^{i}$. 
Consider any output density $2^{d}$ and let $d'=\lfloor d/\sqrt{\log n} \rfloor$. When $2^d$ outputs fire in round $t$, this will excite the first $d'$ 
coarse inhibitors. As a result, the firing probability for the outputs in round $t+1$ will be approximately $2^{-d' \cdot \sqrt{log n}}$ (ignoring negligible effects from the fine inhibitors). In other words, within a single round the density will be reduced from $2^{d}$ to $2^{d-d'\sqrt{ \log n}}$
which is a new density in the range $1, 2, 4, ..., 2^{\sqrt{\log n}}$. 
After this initial round, since at most $2^{\sqrt{\log n}}$ outputs fire, the circuit converges in constant rounds in expectation as the $\sqrt{\log n}$ fine inhibitors can induce probabilities roughly equal to $1/k_t$ just as is done in the $O(\log n)$ inhibitor circuit.

Generalization to larger $\theta$ is by repeating the above construction: we have $\theta$ levels of increasing coarseness: $[1,2^{\log^{1/\theta} n}],  [2^{\log^{1/\theta} n},2^{\log^{2/\theta} n}],...,[2^{\log^{(\theta-1)/\theta} n},2^{\log n}]$. The $\log^{1/\theta} n$ inhibitors at each level ensure that if the number of firing outputs is at level $i$ in round $t$, it is reduced to level $i-1$ in round $t+1$, yielding $O(\theta)$ expected runtime.
We give a full analysis in Appendix \ref{appenx:alpha}.
\end{proof}

Our second construction uses similar techniques, but  
uses just one convergence inhibitor per density class, balancing the time required to move through each density class and the number of classes used. 
It significantly improves on our two inhibitor algorithm, achieving runtime $O(\log^{1/c} n)$ for any constant $c$ with $O(1)$ inhibitors and $O(\log \log n)$ runtime with $O(\log \log n)$ inhibitors. 
\begin{theorem}
\label{thm:upper_symmetrict}
For any $\alpha \ge 2$, there is a basic WTA network $\Net$ with $\alpha$ inhibitors, $\mathcal{ET}(\Net) = O \left (\alpha \log^{1/(\alpha-1)} n\right )$ and $\mathcal{HT}(\Net) = O \left (\alpha \log^{1+1/(\alpha-1)} n\right )$. 
\end{theorem}
\begin{proof}[Proof Sketch]
Consider $\alpha = 3$. We have 2 convergence inhibitors: a fine inhibitor $z_f$ and a coarse inhibitor $z_c$. The inhibitor $z_c$ fires whenever the number of active outputs is at least $2^{\sqrt{\log n}}$, and induces outputs to fire with probability $1/2^{\sqrt{\log n}}$ in the next round. In this way, starting with any density of firing inputs $k_t \in [2^{\sqrt{\log n}},n]$, within $\sqrt{\log n}$ rounds the density will be reduced to $\le 2^{\sqrt{\log n}}$. The inhibitor $z_f$ fires whenever at least $2$ outputs fire, and induces outputs to fire with probability $1/2$ in the next round. So, within $\sqrt{\log n}$ additional rounds, with constant probability just a single output will remain firing.
Again, a full network description for general $\alpha$ and proof is given in Appendix \ref{appenx:alpha}.
\end{proof}
\subsection{Lower Bound: The Tradeoff between Number of inhibitors and Time}\label{sec:mainLower}
We now present our main lower bound which matches Theorem \ref{thm:upper_symmetricalpha} up to $\log \log \log n$ factors.
\begin{theorem}
\label{thm:lbzaconst} For any basic WTA network $\Net$ with $\alpha$ inhibitors,
$\ExpectedT(\Net)=\Omega \left (\frac{\log n \log n}{\log \NumInh} \right)$ and
$\mathcal{HT}(\Net)=\Omega \left (\frac{\log \log n}{\log \alpha}  \cdot \frac{\log n}{\log\log \log n} \right)$.
\end{theorem}
%
\paragraph{Lower Bound Overview.} We focus on initial output configuration $Y^0 = \vec{0}$ (i.e., no output fires in the sub-round $(0,2)$) which we call the
\emph{reset configuration}.
We show that for any network $\Net$ with $\alpha$ inhibitors 
there exists at least one input $\Input$ 
 for which the expected time to reach WTA starting from the reset configuration is $\Omega(\log\log n/\log \alpha)$.
If suffices to consider the case where $\NumInh=O(\log^{1/c} n)$ for some constant $c$ since for $\NumInh=\Theta(\log^{1/c} n)$, the expected runtime is $O(1)$. 
Throughout this section, 
we say an event happens with \emph{good probability} if its probability is at least $1-O(\log^4 n)$.

Our argument contains two main parts. First, we show that the inhibitors fire in a \emph{nearly} deterministic manner and hence we can treat them (up to some slack) as \emph{threshold circuits}. Equipped with this property, we then consider $\Theta(\log n)$ \emph{density classes} each covering a constant multiplicative range of firing outputs. The predictable behavior of the inhibitors is used to show that even after $\Omega(\log n \log n/\log \NumInh)$ rounds, the network cannot distinguish between at least two different density classes, which yields our claim as it does not converge to WTA for at least one class.

\paragraph{(1) Inhibitor classification: inhibitors are nearly deterministic for most density classes.}
To address the first challenge (i.e., showing that inhibitors are predictable), we
divide the set of inhibitors $\Inh$ into three classes and show the predictability property for each class separately. The ``stability" class (or ``WTA preservers") $S$ contains inhibitors whose \emph{goal} is to maintain the WTA steady state.
The ``convergence'' class (or ``progress inhibitors'') $C$ contains the inhibitors that are responsible for driving fast convergence to a WTA state. Finally, the third class $R$ contains the remaining inhibitors whose contribution to both stability and convergence is negligible. 
%

Formally,
for any inhibitor $z_i \in \Inh$ and $j \in [1,n]$ let
$pot_j(z)= j \cdot \weightY_i-b(z_i)$ be the potential of $z_i$ when exactly $j$ outputs fire (I.e., if in sub-round $(t,2)$ the number of firing outputs is $j$, then the potential of $z_i$ in sub-round $(t,3)$ is $pot_j(z)$ and it fires in sub-round $(t,3)$ with probability $1/(1+e^{-pot_j(z)})$). The set $S$ contains all inhibitors that fire in steady state (i.e., when exactly one output is firing) with reasonably high probability. Fixing some constant $c \ge 1$, $S=\{ z_i \in \Inh ~\mid~ 1/(1+e^{-pot_1(z_i)})\geq 1/\log^{3c} n\}$.
The set $C$ is comprised of all inhibitors $z_i \notin S$ whose firing probability is least $1/\log^{c} n$ when all $n$ outputs fire in the previous sub-round: $C=\{z_i \in \Inh ~\mid~ z_i \notin S \text{ and } 1/(1+e^{-pot_n(z_i)})\geq 1/\log^{c} n\}$\footnote{The difference between $1/\log^{3c}n$ when defining the threshold for the inhibitors in $S$ and $1/\log^c n$ when defining the threshold for the inhibitors $C$, is crucial in the analysis.}. Finally, $R$ contains all remaining inhibitors not in $S$ or $C$.

We show that the firing states of the inhibitors can \emph{in certain cases} be predicated with good probability. 
The argument for 
each of the three classes $S,C$ and $R$ is different and is presented in Appendix \ref{append:det}.
Since the inhibitors in $S$ fire with good probability when just one output fires, we can show that they fire \whp when at least two outputs fire:
\begin{lemma}[$S$ is predictable]
\label{lem:safetyinhibitors}
Let $(t,2)$ be a sub-round in which at least two outputs fire, then sub-round $(t,3)$, all inhibitors of $S$ fire with probability at least $1-1/n$.
\end{lemma}
\def\APPENDSAFTEYNOFIRE{
\begin{proof}[\textbf{Proof of Lemma \ref{lem:safetyinhibitors}}]
By the definition of the set $S$, for  $z \in S$ it holds that $z$ fires in sub-round $(t,3)$ with probability $1/(1+e^{-pot_1(z)})\geq 1/\log^{3c} n$ and hence $\weightY_z-b(z)\geq -3c\log\log n$. By our no-background noise assumption that neurons do not fire \whp with no external input, we can assume $b(z) \geq 3 \log n$ and hence have $pot_2(z)=2\weightY_z-b(z) \geq 2\log n $. Thus, $z$ fires with probability at least $1-1/n^2$ in sub-round $(t,3)$. Overall, all the $|S|\leq O(\log n)$ inhibitors fire in sub-round $(t,3)$, with probability at least $1-1/n$ as required.
\end{proof}
}

Since the firing probability of the $R$ inhibitors is small in comparison to the 
$O(\log \log n/\log \alpha)$ execution length that we care about, we have: 
\begin{lemma}[$R$ is predictable]
\label{lem:convergenceinhibitors}
Given any input $\Input$ and any initial configuration, with probability at least $1-1/\log^{c-3} n$, none of the inhibitors in $R$ fire in $O(\log^2 n)$ rounds of execution of $\Net$.
\end{lemma}
\def\APPENDRNOFIRE{
\begin{proof}[\textbf{Proof of Lemma \ref{lem:convergenceinhibitors}}]
In any round $t$, even if all $n$ outputs fire in sub-round $(t,2)$, the firing probability of each inhibitor in $R$ in sub-round $(t,3)$ is at most $1/\log^{c} n$ (or else the inhibitor would fall in $C$). Union bounding over the first $O(\log\log n)$ rounds of execution and the at most $O(\log n)$ inhibitors in $R$, we get that with probability at least  $1-1/\log^{c-3}n$, none of these inhibitors fires in these rounds.
\end{proof} 
}
Perhaps the most surprising claim concerns the predictability of the convergence inhibitors. 
\begin{lemma}[$C$ is almost predictable]
\label{lem:almostdeter} 
For every $z \in C$, there exists an integer $k(z) \in [1,n]$, such that for $c\geq 4$:
\begin{description}
\item{(1) Low Density:}
When there are at most $k(z)/2$ firing outputs in sub-round $(t,2)$, the probability that $z$ fires in sub-round $(t,3)$ is at most $1/\log^{c} n$ (i.e., with good probability, $z$ does not fire);
\item{(2) High Density:}
When there are at least $2k(z)$ firing outputs in sub-round $(t,2)$, the probability that $z$ fires in sub-round $(t,3)$ is at least $1-1/\log^{c} n$ (i.e., with good probability, $z$ fires).
\end{description} 
\end{lemma}
Overall, except for the case where the number of firing outputs in sub-round $(t,2)$ is in the density class $K(z)=[k(z)/2,k(z)]$, $z$ behaves in sub-round $(t,3)$ in an almost deterministic manner. Roughly speaking, this is shown by exploiting the \emph{gap} in the firing probabilities of these inhibitors between the steady state rounds (when they fire with probability $\leq 1/\log^{3c} n$) and the rounds in which there are sufficiently many firing outputs (where they fire with probability $\geq 1/\log^c n$). The proof of Lemma \ref{lem:almostdeter} shows that this gap implies that the sigmoid function which converts the number of firing inputs to $z$'s firing probability must be steep enough such that $z$ has predictable behavior outside a small range around $k(z)$.
\def\APPENDDETL{
\begin{proof}[\textbf{Proof of Lemma \ref{lem:almostdeter}}]
Let $k(z)$ be the smallest integer in $[1,n]$ such that $z$ fires in sub-round $(t,3)$ with probability at least $1/\log^c n$ when $k(z)$ outputs fire in sub-round $(t,2)$. By the definition of $C$, when $n$ outputs fire, $z$ fires in the next sub-round with probability at least $1/\log^c n$, and hence $k(z)$ is well defined. In addition, since $z \notin S$, $k(z)\geq 2$. 

Part (1) of the claim follows immediately by the definition of $k(z)$. To prove part (b), the key idea is to exploit the following gap in the behavior of $z \in C$: since $z$ is not in $S$, the firing probability of $z$ in steady state (with exactly one firing output) is \emph{at most} $1/\log^{3c} n$. On the other hand, when there are at least $k(z)\geq 2$ active outputs, the firing probability of $z$ is \emph{at least} $1/\log^{c} n$. This implies that the sigmoid function which converts the number of firing inputs to $z$'s firing probability must be steep enough such that $z$ fires with good probability when $\ge 2k(z)$ outputs fire. 
By the fact that $z\notin S$, $pot_1(z)=\weightY_z-b(z)\leq -3c\cdot\log\log n$ and so
$\weightY_z \leq b(z)-3c \log\log n~.$
On the other hand, by the definition of $k(z)$, $\weightY_z$ cannot be too small since $pot_{k(z)}(z)=k(z) \cdot \weightY-b(z)\geq -c\cdot\log\log n$ so
\begin{equation}
\label{eq:ssll}
k(z) \cdot \weightY_z \geq b(z)-c \log\log n.
\end{equation}
Combining this we get:
$k(z)b(z)-3k(z)\cdot c\log\log n \geq b(z)-c \log\log n$
and so $b(z)\geq 3c\log\log n$. 
Using that and Eq. (\ref{eq:ssll}), we get:
$pot_{2k(z)}(z)=2k(z)\cdot \weightY-b(z)\geq 2b(z)-2c\log\log n-b(z)
= b(z)-2c\log\log n \geq c\log\log n.$
Hence, $1/(1+e^{-pot_{2k(z)}(z)})\geq 1-1/(\log^c n)$ as required.
\end{proof}
}
\paragraph{(2) Network prediction for nearly deterministic inhibitors:}
Using the predictable nature of the inhibitors, we now show that there is at least one \emph{density class} of competing inputs for which we can predict (with good probability) the behavior of $\Net$ for $\Omega(\log\log n/\log \NumInh)$ rounds, at the end of which the WTA state has not been reached. 
We consider a set of $\ell=\lfloor \log n \rfloor$ inputs $\mathcal{X}=\{\Input_1,...,\Input_\ell\}$ where $\Input_i$ contains exactly $2^{i}$ firing inputs (i.e. $\norm{\Input_i}_1 = 2^i$). Thus, $\mathcal{X}$ contains a representative input from each density class of input vectors whose number of firing inputs is within a factor two of each other.

For any $\Input \in \mathcal{X}$ let $\widehat{R}_t(\Input) \in \{1, \ldots, n\}$ be the random variable indicting the number of firing outputs in sub-round $(t,2)$ starting from the initial configuration $Y_0 = \vec{0}$. Let $\widehat{F}_t(\Input) \in \{0,1\}^{\NumInh}$ be the random variable indicating the firing status of the inhibitors in  sub-round $(t,3)$. 
For each $\Input \in \mathcal{X}$ we will attempt to maintain a \emph{predicted} range $R_t(\Input)$ of the number of firing outputs in sub-round $(t,2)$ along with a \emph{predicted} inhibitor configuration in sub-round $(t,3)$, $F_t(\Input)$.
We will let $\mathcal{X}_t \subseteq \mathcal{X}$ denote the subset of inputs whose behavior we can predict well in (all sub-rounds of) round $t$ -- specifically, for which we know $\widehat{R}_t(\Input) \in R_t(\Input)$ and $\widehat{F}_t(\Input) = F_t(\Input)$ with good probability (at least $1-1/\log n$). 

For any inhibitor $z \in C$, we call the range $K(z) = [k(z)/2, 2k(z)]$-- the \emph{critical range} of $z$ (see Lemma \ref{lem:almostdeter} for the definition of $k(z)$). If the number of firing outputs enters this range, we will not be able to predict the behavior of $z$ in the next sub-round with good probability. On the other hand, as long as the number of firing outputs in sub-round $(t,2)$ is not in the critical range of any $z \in C$, then the firing behavior of the inhibitors in sub-round $(t,3)$ can be predicted with good probability. 

We will progress through rounds, predicting the behavior of $\Net$ in round $t$ for each input in $\mathcal{X}_{t-1}$ based off the predictions in round $t-1$. 
We will ensure that in any round, not too may inputs have predicted ranges overlapping critical regions by ensuring that these predicted ranges remain separated by constant factors and hence, at most $|C|$ of them can overlap $K(z)$ for some $z \in C$.
\paragraph{Predicting the number of firing outputs given inhibitor states:}
We now describe how to predict the range $R_t(X)$ given the prediction $F_{t-1}(X)$.
Our main goal is to preserve the separation between the predicted ranges $R_t(\Input)$ for sufficiently many inputs $\Input \in \mathcal{X}_{t-1}$. 

To maintain the separation, we consider only the largest subset $\mathcal{X}^{same}_t \subseteq \mathcal{X}_{t-1}$ of inputs whose predicted firing configuration for the inhibitors in the previous sub-round $(t-1,3)$ is exactly the \emph{same} (i.e., inputs $\Input$ with the same $F_{t-1}(\Input)$ vector). By doing this, we guarantee that the firing probabilities of all the outputs in sub-round $(t,2)$ is the same. Letting this probability be $p$, the expected number of firing outputs in sub-round $(t,2)$ is in the range $p \cdot  R_{t-1}(\Input)$ for each  $\Input \in \mathcal{X}^{same}_t$ and the separation between these ranges is preserved in expectation. To show that the ranges are also separated with good probability, we omit from $\mathcal{X}^{same}_t$ at most $\Theta(\log\log n)$ inputs with ranges $R_t(\Input)$ containing values $ \le \log^c n$ for some constant $c$. They remaining inputs thus have output ranges concentrated around their expectation. The key point to observe is that because the inhibitors behave almost as threshold circuits, the number of different firing configurations in sub-round $(t-1,3)$ is at most $\NumInh$ (i.e., there are at most $\NumInh$ different $F_{t-1}(\Input)$ vectors for $\Input \in \mathcal{X}_{t-1}$) and hence the cardinality of the set $\mathcal{X}^{same}_t$ for which we predict the range of firing outputs in sub-round $(t,2)$ is at least $|\mathcal{X}_{t-1}|/\NumInh$.
\paragraph{Predicting the inhibitor states given the number of firing outputs:}
We next describe how to predict the inhibitor firings $F_t(\Input)$ given the prediction $R_t(\Input)$.
Since the convergence inhibitors are predictable when the number of firing outputs is not in any critical range $K(z)$, 
we first omit from $\mathcal{X}^{same}_t$ all inputs $\Input$ whose predicted range $R_t(\Input)$ intersects the critical range of some $z \in C$ (i.e. $R_t(\Input) \cap K(z) \neq \emptyset$ for some $z$).
We call the resulting set $\mathcal{X}_t$. Since the ranges of $\mathcal{X}^{same}_t$ are separated by some constant, we do not discard more than $|C| = O(\NumInh)$ inputs.
\par Overall, we predict the circuit behavior in sub-rounds $(t,2), (t,3)$ with good probability for all inputs $\Input \in \mathcal{X}_t$ where $|\mathcal{X}_t| \geq |\mathcal{X}_{t-1}|/\NumInh-\NumInh$. Since $\NumInh=O(\log^{1/c} n)$, we get that after $t$ rounds, there are $|\mathcal{X}_t|=\Omega(\log n/\NumInh^t)$ inputs for which the network behaves \emph{exactly} the same in each of the $t$ rounds with good probability. 
 This argument proceeds as long as $\log n/\NumInh^{t}\geq 2$, leading to the lower bound of expected time $\Omega(\log\log n/\log \NumInh)$ since we can show if two inputs are not distinguished, at least one will not have reached WTA.
In Appendix \ref{Append:LBEDetailed}, we describe the prediction process in detail and complete the proof of Theorem \ref{thm:lbzaconst}.

\def\APPENDDetLBE{
In this section we describe the prediction process in more detail. 

\paragraph{Inductive Assumptions:}
For each round $t$, in showing that we are able to predict the behavior of $\Net$ for a large number of inputs in round $t$, we make several inductive assumptions:

For two ranges of positive numbers $R_1=[r_1,r_2]$ and $R_2=[r_3,r_4]$ such that $r_1\leq r_2 \leq r_3 \leq r_4$, and a positive number $a$, the ranges are called $a$-separated if $r_3/r_2\geq a$.  The \emph{value} of the range $R_1=[r_1,r_2]$ is taken to be $r_1$.
We assume that for $\Input \in \mathcal{X}_{t-1} \subset \mathcal{X}$ the ranges $R_{t-1}(\Input)$ are all $a$ separated for some constant $a$ and have minimum value $\Theta(\log^7 n)$. We also assume that our earlier predictions are accurate: for each $X \in \mathcal{X}_{t-1}$, $\widehat{R}_{t-1}(\Input) \in R_{t-1}(\Input)$ and $\widehat{F}_{t-1}(\Input) = F_{t-1}(\Input)$ with probability at least $1-\Theta(1/\log n)$. 
We first show that these assumptions hold for round one:

\paragraph{Predicting the number of firing outputs in sub-round $(t=1,2)$.}
Since we consider the initial reset configuration $Y^0 = \vec{0}$ we have $\widehat{R}_0(\Input_i) = 0$ for all $\Input_i$. Trivially we can set $\mathcal{X}_0 = \mathcal{X}$ -- we deterministically know the behavior of all outputs in round $0$.
By our no-background noise assumption, for every $z \in \Inh$, 
$b(z)=c\log n$, and so w.h.p. $\widehat{F}_0(\Input_i) = 0$ for all $\Input_i$ (no inhibitor fires in the initialization round). Let $\mathcal{X}_1^{large} = \{\Input_i ~\mid~ 2^i \ge \log^9 n\}$
(note that $|\mathcal{\Input}_1^{large}|=\Theta(\log n)$). 
Let $p_0$ be the probability that an output fires in sub-round $(t+1,2)$ given that no inhibitor and no output fires in round $t$ (i.e, no output has an active self-loop). Since there are $2^i$ active \emph{input neurons} in $\Input_i$, conditioned on the high probability event that $\widehat{R}_0(\Input_i) = 0$ and $\widehat{F}_0(\Input_i)=\vec{0}$, the expected number of firing outputs in sub-round $(1,2)$ is $p_0 \cdot \Input_i$. It is not hard to show that $p_0 = \Omega(1/\log^2 n)$ 
and by combining this fact with a Chernoff bound we have: 
\begin{claim}
\label{cl:first_round}
For every $\Input_i \in \mathcal{X}^{large}_1$, \whp the number of firing outputs in sub-round $(1,2)$, $\widehat{R}_1(\Input_i)$ is in the range $R_1(\Input_i) =  [(1- 1/\log^{3} n) \cdot p_0 2^i, (1+1/\log^{3} n) \cdot p_0 2^i]$. Hence, the predicted output ranges for the inputs in $\mathcal{X}^{large}_1$ are $2(1-1/\log n)$ separated. Additionally each has minimum value $\Omega(\log ^7 n)$.
\end{claim}
\def\APPENDZEROSTATE{  
\begin{proof} 

Let $\Input_1$ be a vector with exactly one firing input and let $y_i$ be its corresponding output. Starting from $Y^0 = \vec{0}$, w.h.p., no inhibitor fires in round $0$. If 
$p_0< 1/\log^2 n$ then since $p_0$ rate is the maximum firing probability for $y_j$ in sub-round $(t+1,2)$ given that it didn't fire in sub-round $(t,2)$, the network requires $\Omega(\log^2 n)$ rounds until $y_j$ fires with constant probability and so at least that long to converge to WTA. So we can work in the case where $p_0 \geq 1/\log^2 n$.

For $\Input_i \in \mathcal{X}^{large}$ we thus have the expected number of firing outputs in sub-round $(1,1)$ is $p_0 \cdot 2^i \ge 1/\log^2 n \cdot \log^9 n = \log^7 n$.
Since the random firings of the outputs are independent given the firing behavior of the inhibitors and since no inhibitors fire in sub-round $(0,3)$ \whp by a Chernoff bound (Theorem \ref{thm:simplecor}), we have that \whp the number of firing outputs $\widehat{R}_1(\Input_i)$ is in the range $(1\pm 1/\log^{3} n) \cdot p_0 \cdot 2^i$ for all $\Input_i \in \mathcal{X}^{large}$.
\end{proof}
}
\APPENDZEROSTATE  

The above shows that the predicted ranges for all $\Input \in \mathcal{X}^{large}_1$ are well separated, accurate, and have high value.
We can now set $\mathcal{X}_1$ to include any $\Input \in \mathcal{X}_1^{large}$ except possibly $|C| \le \alpha$ inputs where $R_1(\Input)$ overlaps a critical region $K(z)$ for some $z \in C$. Since the remaining ranges do not overlap any critical regions, by Lemmas \ref{lem:safetyinhibitors}, \ref{lem:convergenceinhibitors}, and \ref{lem:almostdeter} we are able to predict $\widehat{F}_1(\Input)$ with good probability, and so have all our inductive assumptions in round $1$.

\paragraph{Predicting the number of firing outputs for rounds $t \ge 2$.}

We first  define a subset of inputs $\mathcal{X}^{large}_t \subseteq \mathcal{X}_{t-1}$ for which we can predict the behavior of the outputs in $\Net$ in sub-round $(t,2)$.
Let $\mathcal{X}^{same}_t\subseteq \mathcal{X}_{t-1}$ be the largest subset of inputs whose predicted firing vector $F_{t-1}(\Input)$ for the inhibitors in sub-round $(t-1,3)$ is the \emph{same}, and denote this common firing vector by $F^*_{t-1}$. 
Let $\mathcal{X}^{large}_t$ be the set of inputs in $\mathcal{X}^{same}_t$ after omitting $\Theta(\log\log n)$ inputs with the smallest range value in sub-round $(t-1,2)$.

Eventually we will show that $\mathcal{X}^{large}_t$ is a reasonably large set of inputs compared to $\mathcal{X}_{t-1}$, and hence we can continue predicting behavior for at least some inputs for a large number of rounds. But first we show how to  predict $R_{t}(\Input)$ for 
every input $\Input \in \mathcal{X}^{large}_t$.

Let $p$ be the probability that an active output (one with $y_j^{(t-1,2)} = 1$) fires in sub-round $(t,2)$ given that the inhibitors fired in sub-round $(t-1,3)$ according to $F^*_{t-1}$. Since all inputs in $\mathcal{X}^{same}_{t}$ have the same predicted firing vector $F^*_{t-1}$, in each of them, an active output 
fires in sub-round $(t,2)$ with probability $p$. In addition, by induction for every $\Input \in \mathcal{X}_t^{same} \subseteq \mathcal{X}_{t-1}$, $R_{t-1}(\Input)$ has a minimum of $\Theta(\log^7 n)$ predicted firing outputs. So inhibition in sub-round $(t-1,3)$ \whp must be at least as high as it is once we have converged to WTA and just a single output is firing. Thus, any output that did not fire in sub-round $(t-1,2)$ must not fire \whp in sub-round $(t,2)$, since non-firing outputs continue not to fire once WTA is converged to. 

So just focusing on active outputs that fire in sub-round $(t,2)$,
for every $\Input_i \in \mathcal{X}^{same}_{t}$, let $R_{t-1}(\Input_i)=[\ell_{i},m_{i}]$ be the predicated range of firing outputs in sub-round $(t-1,2)$. Then the expected number of firing outputs in sub-round $(t,2)$ is in the range $[p\cdot \ell_{i}~,~ p \cdot m_{i}].$ For every $\Input_i \in \mathcal{X}^{same}_{t}$, let 
$R_{t}(\Input_i)=[(1-1/\log^3 n)\cdot p\ell_{i}~,~(1+1/\log^3 n) \cdot pm_{i}].$

We now make the following observation that states that if the expected number of firing outputs is too small for even one of the inputs in $\mathcal{X}^{same}_{t}$, then it implies a lower bound of $\Omega(\log n)$ for $\mathcal{ET}(\Net)$.
Essentially this is because if this is the case, with good probability, $0$ outputs will fire in round $t$, and a reset configuration identical to $Y^0$ will occur. This will keep occurring, causing the network to have large runtime. The proof appears in Appendix \ref{appenx:lb}. 
\begin{observation}
\label{Obs:resetorcontinue}
For every $t\geq 1$, if there exists $\Input \in \mathcal{X}^{same}_{t}$, such that the smallest value of $R_t(\Input_i)$ is less then $1/\log^4 n$, then $\ExpectedT(\Net)=\Omega(\log n)$.
\end{observation}
\def\APPENDROC{
\begin{proof}
Let $\Input \in \mathcal{X}^{same}_{t}$ be such that $R_t(\Input)$ is less then $1/\log^4 n$.
Then, given that the inhibitors fire according to the prediction $F^*_{t-1}$ in sub-round $(t-1,3)$, by Markov inequality, the probability that the number of firing outputs in sub-round $(t,2)$ is at least $1$ is less then $1/\log^4 n$. In other words, the conditional probability (where we condition on the prediction for round $t-1$) that a reset where $0$ outputs fire happens in sub-round $(t,2)$ is at least $1-1/\log^4 n$. However, by our inductive assumption $\widehat{F}(\Input) = F^*_{t-1}$ must be correct with probability at least $1-1/\log n$. Hence, with probability at least $1-\Theta(\log n)$ $Y^t = \vec{0}$ and a reset round occurs. With constant probability this occurs $\Omega(\log n)$ times before WTA is ever reached. The observation follows. 
\end{proof}
}
\APPENDROC

Hence, from now on, we assume the complementary case that the number of predicted firing outputs in sub-round $(t,2)$ is at least $1/\log^4 n$ for every $\Input \in \mathcal{X}^{same}_{t}$.  This allows us to show:
\begin{claim}
\label{cl:noresetgoodconc} 
For every $\Input \in \mathcal{X}^{large}_{t}$
\begin{description}
\item{(1)}
Given that the inhibitors fire according to $F^*_{t-1}$ in sub-round $(t-1,3)$, then with probability $1-1/n$, the number of firing outputs in sub-round $(t,2)$ is in the range $R_{t}(\Input)$.
\item{(2)}
The set of ranges $R_{t}(\Input)$ for $\Input \in \mathcal{X}^{large}_{t}$ are all $a$-separated for some constant $a$. 
\item{(3)} $R_{t}(\Input)$ has value at least $\Omega(\log^7 n)$ for every $\Input \in \mathcal{X}^{large}_{t}$.
\end{description}
\end{claim}
\begin{proof}
Since for any $\Input \in \mathcal{X}^{same}_{t}$ the predicted number of firing outputs is $\Omega(1/\log^4 n)$, and since the ranges are constant separated by our inductive assumption that  the ranges $R_{t-1}(\Input)$ for $\Input \in \mathcal{X}_{t-1}$ are separated, by omitting $\Theta(\log\log n)$ inputs from $\mathcal{X}^{same}_{t}$, the minimum number of firing outputs in the predicted ranges for the remaining set of inputs, namely, $\mathcal{X}^{large}_{t}$ is $\Omega(\log^7 n)$. Hence the true number of firing outputs is well concentrated around this expectation and so we have (1) by a Chernoff bound (Theorem \ref{thm:simplecor}).

Further, since we increase the width of the predicted range $R_{t}(\Input)$ by factor of at most $(1+1/\log^3 n)$ compared to the range $R_{t-1}(\Input)$, over all
$O(\log \log n)$ rounds of prediction, the range is increased by at most a factor of $(1+1/\log^3 n)^{O(\log \log n)} \le 1+O(1/\log^2 n)$. Since the ranges have separation $2$ in the initialization round, they remain constant separated in round $t$, giving (2).
\end{proof}

\paragraph{Predicting $\widehat{F}_{t}(\Input)$ given the predicted range $R_{t}(\Input)$.}
We first define the final subset $\mathcal{X}_{t} \subseteq \mathcal{X}^{large}_{t}$ of inputs for which round $t$ is fully predicted (i.e., both the number of firing outputs in sub-round $(t,2)$ and the states of the inhibitors in sub-round $(t,3)$). 
The set $\mathcal{X}_{t}$ contains any $\Input \in \mathcal{X}^{large}_{t}$ unless $R_{t}(\Input)$ intersects the critical range $K(z)$ for some convergence inhibitor $z \in C$. 
By Lemma \ref{lem:almostdeter}, the firing state of each inhibitor $z \in C$ can be predicted with good probability as long as the number of firing outputs in previous sub-round is not in the critical range $K(z) = [k(z)/2,2k(z)]$. In particular, if the range $R_{t}(\Input)$ falls below $k(z)/2$, then we predict that $z$ does not fire in sub-round $(t,3)$. On the other hand, if the range $R_{t}(\Input)$ falls above $2k(z)$, then we predict that $z$ fires in sub-round $(t,3)$. Regardless of the exact number of firing outputs in sub-round $(t,2)$, since $R_{t}(\Input)$ does not intersect the critical ranges of the inhibitors of $C$, we can predict with good probability the firing states of $C$ in sub-round $(t,3)$ by Lemma \ref{lem:almostdeter}. By Lemma \ref{lem:safetyinhibitors}, with probability at least $1-1/n$, all the stability inhibitors $S$ fire in sub-round $(t,3)$ and by Lemma \ref{lem:convergenceinhibitors}, with good probability, no inhibitor in $R$ fires. So overall we can predict all inhibitor behavior with good probability.
 With the above in place we are finally have that our inductive assumptions hold in round $t$. We summarize:
\begin{lemma}
\label{lem:properties}
For every $t\geq 1$ it holds that:
\begin{description}
\item{(Q1)}
For every $\Input \in \mathcal{X}_t$, the predicted range of firing outputs $R_{t}(\Input)$ satisfies:
\begin{equation}
\label{eq:rguarntee}
\Pr[\widehat{R}_{t}(\Input)\in R_{t}(\Input)~\mid~ \widehat{F}_{t-1}(\Input)=F_{t-1}(\Input)]\geq 1-1/n~. 
\end{equation}
\item{(Q2)}
The collection of predicted ranges $R_{t}(\Input)$ for $\Input \in \mathcal{X}_{t}$ are all $a$-separated for some constant $a$ and all have value at least $\Omega(\log^7 n)$.
\item{(Q3)}
For every $\Input \in \mathcal{X}_t$, the predicted firing pattern for the inhibitors satisfies
\begin{equation}
\label{eq:fguarntee}
\Pr[\widehat{F}_{t}(\Input)=F_{t}(\Input)~\mid~ \widehat{R}_{t}(\Input) \in R_{t}(\Input)]\geq 1-1/\log^3 n~. 
\end{equation}
\end{description}
\end{lemma}
%
%
%


The final step before giving our expected time lower bound is to show that $\mathcal{X}_t$ is reasonably large, so we are able to keep predicting the behavior of $\Net$ for a number of outputs round after round. This follows from a few simple observations:

\begin{observation}
\label{obs:sizecollecj}
$|\mathcal{X}^{same}_t|\geq |\mathcal{X}_{t-1}|/\NumInh$. 
\end{observation}
Recall that $\mathcal{X}_t^{same}$ consists of the largest subset of $\mathcal{X}_{t-1}$ with the \emph{same} predicted inhibitor behavior $F^*_{t-1}$ in round $t-1$.
Naively, there are $2^{\alpha}$ possible predictions for $F^*_{t-1}$ which gives that $|\mathcal{X}^{same}_t | \ge |\mathcal{X}_{t-1}|/2^{\NumInh}$. In order to obtain the much stronger bound above, we again use Lemma \ref{lem:almostdeter} which shows that, as long as $\widehat{R}_{t-1}(\Input)$ does not intersect the critical region of any $z \in C$, the inhibitors behave with good probability as linear threshold circuits and so there are only $\alpha$ possible predictions $F_{t-1}(\Input)$.
\def\APPENDCOLLECJ{
\begin{proof}
Since by Lemma \ref{lem:almostdeter} each inhibitor $z \in C$ behaves with probability $1-\log^c n$ as a threshold network in sub-round $(t,3)$ (so long that the number of firing outputs in sub-round $(t,2)$ is not in the critical range $K(z)$), the total number of different inhibitor firing state configurations (different $F_{t-1}(\Input)$ vectors predicted in the previous step) is bounded by $|C|$. To see this, since conditioning on the prediction $R_t(\Input)$ being correct, there is at least one firing output in sub-round $(t-1,2)$,
the inhibitors of $S$ will fire \whp Further the inhibitors $R$ never fire with good probability, so the only varying part in $F_{j-1}(\Input)$ is the prediction for $C$ and as discussed there are only $|C| \le \alpha$ such possible predictions. 
\end{proof}
}
\APPENDCOLLECJ

\begin{observation}
\label{obs:largesame}
$|\mathcal{X}_t^{large} |\geq |\mathcal{X}^{same}_{t}|-O(\log \log n)$~.
\end{observation}
This is immediate as $\mathcal{X}_t^{large}$ was derived by removing $\Theta(\log \log n)$ of the inputs with the smallest predicted range values from $\mathcal{X}^{same}_{t}$.

\begin{observation}
\label{obs:largesame2}
$|\mathcal{X}_t|\geq |\mathcal{X}^{large}_{t}|-O(\alpha)$~.
\end{observation}
This follows as $\mathcal{X}_t$ is derived by removing all inputs from $\mathcal{X}^{large}_{t}$ where $R_t(\Input)$ overlaps the critical region of some $z \in C$. By (Q2) the $R_t(\Input)$ are all constant separated so there can be at most $|C| = O(\alpha)$ which overlap critical regions.
%
We are now ready to show: 
\begin{lemma}
\label{lem:lbconstant}
$\mathcal{ET}(\Net)=\Omega(\log\log n/\log \NumInh)$.
\end{lemma}
\begin{proof}
We can continue predicting the behavior of $\Net$ up to round $t$ until we have 
$|\mathcal{X}_{t}| = \Theta(\log\log n)$ (at which point $\mathcal{X}_t^{large}$ may be empty and so we will have to stop simulation).  Further, as long as we can predict for $t$ rounds, by Lemma \ref{lem:properties} we will know with good probability that at least $\Omega(\log^7 n)$ outputs are still firing for all $\Input \in \mathcal{X}_t$. So with good probability WTA is not reached for those inputs, giving a lower bound of $\Omega(t)$ rounds in expectation to solve WTA.

Set $t = c_1\log \log n/\log \alpha$ for small enough constant $c_1$ and recall that we can assume $\alpha = O(\log^{c_2} n)$ for small constant $c_2$ since otherwise our runtime bound is $\Omega(1)$ and so holds vacuously.
By Observations \ref{obs:sizecollecj}, \ref{obs:largesame}, and \ref{obs:largesame2} after $t$ rounds we have:
\begin{align*}
|\mathcal{X}_{t}| &\ge \frac{|\mathcal{X}_{0}|}{\alpha^t} - t\cdot \alpha -t \cdot O(\log \log n)\\
&\ge \frac{\log n}{\log^{c_1} n} - \log \log n \cdot \log^{c_2} n - (\log \log n)^2 = \Omega (\log^{1-c_1} n)
\end{align*}
and hence can predict for at least $t$ rounds. This completes the proof.
\end{proof}
}

\paragraph{High Probability Lower bound.}
Finally, we show that our lower bound for expected runtime extends to a lower bound on the high probability runtime. Our lower bound implies that ``repeating" the execution of a network that converges with constant probability $\Theta(\log n)$ times to achieve a high probability guarantee is essentially the best one can do (up to a $\log \log \log n$ factor).
\begin{lemma}
\label{lem:highprobfromcprob}
For any basic WTA network $\Net$ with $\alpha$ inhibitors 
$\mathcal{HT}(\Net)=\Omega(\frac{\log n\cdot \log\log n}{\log \NumInh \log\log\log n})$.
\end{lemma}
\begin{proof}[Proof Sketch]
Let $DC=\Theta \left (\frac{\log \log n}{\log \NumInh}\right)$ and $DH=DC \cdot \left (\frac{\log n}{\log\log\log n} \right )$. 
Fix a network $\Net$ with $\NumInh$ inhibitors and let $\Input$ be the input for which, by Theorem \ref{thm:lbzaconst}, $\Net$ requires at least $DC$ rounds in expectation starting from initial configuration $\mathcal{C}_0$ with input $\Input$ and $\Output^0 = \vec{0}$. In the following proof, we will actually exploit the fact that the lower bound in Theorem \ref{thm:lbzaconst} applies to the time it takes to reach a WTA state with \emph{constant probability} (a stronger time measure than expected time).

We work with the \emph{execution tree} $T$ which includes all possible $DH$ round executions of $\Net$ starting from $\mathcal{C}_0$. The tree $T$ has depth $DH$ where each layer corresponds to the configuration of the network in each round $t$. Each node $u$ at level $t$ is labeled by an $(n+\NumInh)$-length binary vector $Q(u)$ describing the firing states of the outputs and inhibitors in round $t$, i.e., the firing states of the outputs in sub-round $(t,2)$ and the firing states of the inhibitors in sub-round $(t,3)$. Node $u$ has $2^{n+\alpha}$ children, with the edge to each child labeled with the transition probability between the configuration in $u$ to the child configuration. The root node $r$ is labeled with $\mathcal{C}_0$. The mass of node $u$ is given by the product of edge weights on its path to $r$. It is the probability of reaching $u$'s configuration through that execution path.
We call a node $u$ a \emph{reset} node (resp., \emph{WTA} node), if in the configuration $Q(u)$ no output fires (resp., exactly one output with active input fires). 

In order to lower bound $\mathcal{HT}(\Net)$ we will show that the probability to reach a non-WTA leaf node when starting from the root $r$ is at least $1/n^2$, and thus the probability to reach a WTA leaf node is at most $1-1/n^2 < 1-1/n^c$, contradicting a high probability runtime of $\le DH$ rounds.  

Our strategy is based on traversing the tree in an asynchronous manner from the root to (sufficiently many) non-WTA leaf nodes with sufficiently high total probability mass. For a given node $u$ in layer $t$, we may move to a subset of its \emph{non-WTA} children nodes in layer $t+1$. We call this move a \emph{small jump}. Alternatively, we may make a \emph{large jump}, moving $DC$ steps from $u$ and proceeding the traversal from a subset of \emph{non-WTA} leaf nodes of $T_{DC}(u)$ (the height $DC$ subtree rooted at $u$). 
With each jump starting at $u$, we loose some probability mass -- the idea is to show that we do not loose it too quickly.

In more detail, in each step of our traversal, we maintain a collection of non-WTA nodes. 
When arriving a node $u$ in the traversal, we consider its configuration $Q(u)$ and look at the probability that the next round is a \emph{reset} round (with $0$ firing outputs) given $Q(u)$. We show that if the probability of having at most $1$ firing outputs in the next round is $\geq 1/\log\log n$, the probability of having a reset (no firing outputs) is large -- i.e., $\geq 1/(\log\log n)^3$. 

In this case we continue traversal only from the children of $u$ that are reset nodes.
For each of these children $v$, let
$T_{DC}(v)$ be the execution tree of depth $DC$ rooted at $v$. By the lower bound in Theorem \ref{thm:lbzaconst},
the probability to reach a non-WTA leaf node in $T_{DC}(v)$ starting from $Q(v)$ is at least a constant. So from each reset-node $v$, we make a large jump to the leaves of $T_{DC}(v)$. Overall, we maintain a $\Theta(1/(\log\log n)^3)$ fraction of the probability mass of $u$ in making this large jump. 
Since such a jump can occur at most $DH/DC=\log n/\log \log \log n$ times, we maintain at least a $1/(\log\log n)^{3DH/DC} \ge 1/n^2$ fraction of the probability mass throughout the traversal.

On the other hand, when arriving a node $u$ for which the probability of having at most $1$ firing output in the next round is less than $1/\log\log n$, we make a small jump to the children nodes of $u$ in which the number of firing outputs is at least $2$ (and hence which are non-WTA nodes). This jump maintains $1-1/\log\log n$ of the probability mass and since such a jump can happen at most $DH$ times, overall we again maintain $(1-1/\log\log n)^{DH} \ge 1/n^2$ of the original probability.

Overall, through making both large and small jumps, at the end of the traversal, we reach a set of non-WTA nodes containing at least a $1/n^2$ fraction of the probability mass in the $DH$ level execution tree. This gives us our high probability time lower bound.
In Appendix \ref{appenx:lbhp} we provide a complete analysis. See Figure \ref{fig:simulationtree} for an illustration of the execution tree.

%
%
\end{proof}

Finally, In Appendix \ref{sec:excitat}, we extend our lower bounds (for both expected and high probability time) to the case where the $\NumInh$ auxiliary neurons can be both excitatory and inhibitory neurons. This holds under the restriction that outputs with no active input are not allowed to fire during the execution. Only competing outputs (that have a positive signal from their inputs) ever fire.
\def\APPENDHIGHLB{
Let $Q_{Y} \subseteq \{0,1\}^{n}, Q_Z \subseteq \{0,1\}^{\NumInh}$ be the vectors describing the firing states of the outputs and inhibitors in a given round. Let $Q=Q_Y \circ Q_Z \subseteq \{0,1\}^{n+\NumInh}$ be a vector describing the firing states of the inhibitors and outputs.
Let $P_{1,j}(Q)$ be the probability to achieve the WTA state in round $j$ given $Q$, that is the probability that exactly one output fires in sub-round $(j,2)$ given that the firing states of the outputs (resp., inhibitors) in sub-round $(j-1,2)$ (resp., $(j-1,3)$) is $Q_Y$ (resp., $Q_Z$). Similarly, let $P_{0,j}(Q)$ be the probability that \emph{no output} fires in sub-round $(j,2)$ given $Q$, that is the probability that a reset event happens. Finally, let $P_{01,j}(Q)$ be the probability that a reset event or a WTA event happens in round $j$ given that configuration in round $j-1$ is $Q$, hence $P_{01,j}(Q)=P_{1,j}(Q)+P_{0,j}(Q)$.
We begin by claiming the following. 
\begin{claim}
\label{cl:probclose}
For every round $j$ and for every vector $Q \in \{0,1\}^{n+\NumInh}$ in which there are at least two firing outputs (i.e., $Q$ is neither a WTA state nor a reset state), and such that $P_{01,j}(Q)\geq \Theta(1/\log\log n)$, it holds that $P_{0,j}(Q)\geq \Theta(1/(\log\log n)^3)$.
\end{claim}
\begin{proof} 
Since $P_{01,j}(Q)=P_{0,j}(Q)+P_{1,j}(Q)$, if $P_{0,j}(Q) \geq P_{01,j}(Q)/2$, then we are done. Hence, we can assume from now on that $P_{1,j}(Q)=\Theta(1/\log\log n)$. We will show that $P_{0,j}(Q)\geq P_{1,j}(Q)/(\log\log n)^2$, which will establish our claim.

Let $p$ be the firing probability of an active output\footnote{Recall that an output is active in round $j$ if it fires in sub-round $(j-1,2)$.} in sub-round $(j,2)$ given $Q$ and let $k\geq 2$ be the number of outputs that fire in round $j-1$ as specified by $Q$. Since $Q$ has at least two firing outputs, w.h.p., only active outputs (those that fire in the previous round) can fire in the next round. The probability that the WTA state is achieved in round $j$ is $P_{1,j}(Q)=k \cdot p \cdot (1-p)^{k-1}$ and the probability that a reset is achieved in round $j$ is $P_{0,j}(Q)=(1-p)^{k}$. 

We consider two cases depending whether the firing probability $p$ is large or small. First, assume that $p \geq 0.1$ and set $r=c/\log\log n$. Since $P_{1,j}(Q)\geq r$, we have that $1-p\geq r/k$. 
We also have: 
$$k(9/10)^{k-1}\geq k \cdot p \cdot (1-p)^{k-1} \geq r,$$ and hence $k\leq \Theta(\log\log n)$. 
Overall, $P_{0,j}(Q)/P_{1,j}(Q)=(1-p)/(kp)\geq (1-p)/k\geq r/k^2 \geq c/(\log\log n)^2$. 
Next, consider the complementary case where $p<0.1$. 
Letting $y=kp/2$, we get $$y \cdot e^{-y} \geq (kp/2)(1-p)^{k/2} \ge (k/2)p(1-p)^{k-1} \geq r/2,$$
hence $y \leq 2\log 1/r=\Theta(\log\log\log n)$. Overall, $P_{0,j}(Q)/P_{1,j}(Q)=(1-p)/kp \geq \Theta(1/\log\log\log n)$. 
\end{proof}

\paragraph{The Execution Tree.}
A key tool used in this section is the notion \emph{execution tree} that captures all possible transcripts that can evolve in a window of $DH$ rounds when starting with the initial configuration $C_0$. The execution tree $T$ is a tree of depth $DH$ where each layer $j$ corresponds to round $j$ when running the network on the initial configuration $C_0$. Each node in $T$ 
is labeled by an $(n+\NumInh)$-length binary vector describing the firing configurations (or states) of the outputs and the inhibitors in a given round, and the edges are labelled by the transition probabilities. Hence, this tree describes all the possible firing states in a span of $DH$ rounds when starting from the initial configuration $C_0$ (for which the time it takes to achieve WTA with constant probability is at least $DC$). The root $r$ is labeled by the zero vector (since in round $0$, no output fires and hence w.h.p also no inhibitor fires). For every $j \geq 2$, every node $u$ in layer $j$ is labeled by a vector $Q(u)=Q_Y(u)+Q_Z(u)\in \{0,1\}^{n+\NumInh}$ describing the firing status of the outputs and the inhibitors in round $j$. Hence, each node has $2^{n+\NumInh}$ children in the configuration tree. 
Every edge $e=(\pi(u),u)$ connecting $u$ to its parent $\pi(u)$ in $T$ is labeled by a probability $p(e)$ that the firing configuration in round $j$ is $Q(u)$ given that the configuration in round $j-1$ is $Q(\pi(u))$.

Let $T_d(u)$ be the subtree of depth $d$ rooted at $u$. When $d$ is omitted $T(u)$ is simply the entire subtree of $u$ in $T$.

For a leaf node $\ell \in T$, let $\mathcal{P}(\ell)=[r=u_{0}, u_{1}, \ldots, u_{DH}]$ be the path connecting $\ell$ to the root $r$ in $T$. Let $p_{leaf}(u)$ be the probability that starting from $r$ the firing configuration in each round $j \in \{0, \ldots, DH\}$ is $Q(u_j)$. Since there is an independence between the coin flips in every round $j$ given the configuration in round $j-1$, we get that
\begin{eqnarray*}
\label{eq:leafnodeprob}
p_{leaf}(u)&=&\prod_{j=0}^{DH} \Pr[Q_Y(u_j) \mbox{~in sub-round~} (j,2)~\mid~ Q_Y(u_{j-1}),Q_Z(u_{j-1}) \mbox{~in sub-rounds~} (j-1,2),(j-1,3)]
\\&\cdot& \Pr[Q_Z(u_j) \mbox{~in sub-round~}(j,3) ~\mid~ Q_Y(u_j) \mbox{~in sub-round~} (j,2)]
\\& =&
\prod_{j=1}^{DH} p(e_j) \mbox{~~where~~} e_{j}=(u_{j},u_{j+1}).
\end{eqnarray*}
For a node $u \in T$, let $Leaf(u)$ be the set of leaves in $T(u)$ and define
\begin{equation}
\label{eq:leafnodeprob}
p_{node}(u)=\sum_{\ell \in Leaf(u)} p_{leaf}(u)~,
\end{equation}
and for a subset of nodes $U$, let $p_{node}(U)=\sum_{u \in U}p_{node}(u)$.
It is convenient to view $p_{node}(u)$ as the \emph{weight} of tree $T(u)$. Hence, the weight of $T$ is $1$. In the same spirit, for a given subset of nodes $U_i$ whose subtrees in $T$ are vertex disjoint, we view $\sum_{u \in U_i}p_{node}(u)$ as the weight of the forest $\bigcup_{u \in U_i} T(u)$.
We would like to show that:
\begin{equation}
\label{eq:leafnodeprob}
\sum_{u \in Leaf(r)}\{p_{leaf}(u) ~|~ u \mbox{~~is a WTA node~}\}<1-1/n^c~,
\end{equation}
In the next paragraphs, we will find a collection of  non-WTA leaf nodes of large weight, i.e. of weight at least $1/n^2$  which will establish Eq. (\ref{eq:leafnodeprob}) for $c>2$.
To do that, we iteratively traverse the tree $T$ from root to leaves, omitting undesired subtrees (and hence also leaf nodes) through the journey. This traversal is done in an asynchronous manner in the following sense: there are times that for a given node $u$ in layer $j$, we move to a subset of its children in layer $j+1$, we call this move a \emph{small jump} in the tree. In contrast, there are cases in which from a given node $u$ in layer $t$, we jump $DC$ layers in the subtree $T(u)$ and proceed the traversal from a subset of leaf nodes in the tree $T_{DC}(u)$ of depth $DC$, we call such a move a \emph{large jump}. In the analysis part we will claim that by eliminating nodes in the tree $T$, we do not loose much weight, to deal with the fact that there are two types of jumps: small and large, we will employ an amortization claim that will enable us to bound the loss of weight layer by layer. See Fig. \ref{fig:simulationtree}, for an illustration of the Execution Tree. 

In each iteration $j \in \{1, \ldots, DH\}$, we maintain a collection of \emph{non-WTA} nodes $U_j$ whose subtrees in $T$ are vertex disjoint.
The final set $U_{DH}$ will be a set of non-WTA leaf nodes for which we will show that their weight is at least $1/n^2$. Starting with $U_0=\{r\}$, in every iteration $j \in \{1, \ldots, DH\}$, we have a set of nodes $U_j$ that satisfy the following:
\begin{description}
\item{(A1)}
The subtrees $T(u)$, $u \in U_j$, are vertex-disjoint. 
\item{(A2)}
The distance of each node $u \in U_j$ from $r$ is at least $j$.
\item{(A3)}
No node in $U_j$ is a WTA node.
\end{description}
In the high level, the nodes $U_{j+1}$ are the leaf nodes of subtrees  rooted at the nodes $u \in U_j$. Particularly, from each node $u \in U_j$, when constructing $U_{j+1}$, we omit part of the subtree $T(u) \subseteq T$ and replace $u$ by a subset of nodes $V(u)$ in the subtree of $u$ in $T$. The nodes $V(u)$ are subset of the leaf nodes of the subtree $T_{d(u)}(u)$ of depth $d(u)$ rooted at $u$. The value of the depth $d(u)$ is set to be either $1$ or $DC$ \footnote{To be more precise it is either $1$ or $\min\{DC, DH-dist(r,u,T)\}$.} depending on the configuration stored at node $u$. That is, either the nodes $V(u)$ are a subset of the children of $u$ or that they are subset of the leaf nodes of the $DC$-depth tree rooted at $u$. 

In the first case where $d(u)=1$, we will show that we loose only $\Theta(1/\log\log n)$ of the weight of the tree $T(u)$, hence we keep $1-\Theta(1/\log\log n)$ fraction of the weight. In the second case, we will show that we keep $\Theta(1/\log\log n)$ fraction of the weight of $T(u)$. The key observation here is to note that this cannot happen more than $DH/DC$ times in a given branch, since the depth of the sub-tree of $u$ is $DC$. In other words, on average, we maintain $\Theta(1/\log\log n)^{1/DC}$ of the weight per layer of the subtree $T_{DC}(u)$, and hence overall, after $DH$ iterations, we maintain $1/n^2$ fraction of the total weight. 

We first eliminate from the tree $T$ all nodes $u$ such that $Q_Y(u)=\vec{0}$ but $Q_Z(u)\neq \vec{0}$. Since the bias value of the inhibitors in $\Omega(\log n)$, we know that if no output fires in round $j$, then w.h.p. no inhibitor fires in that round. Let $T'$ be the resulting tree. We first observe that by that step, we eliminate only $1/n^c$ of the total weight of the tree $T$.
\begin{observation}
The total weight of $r$ in $T'$ is at least $1-1/n^c$. 
\end{observation}
From now on, we consider the tree $T'$ and describe the iterative construction of the set $U_j$ in details. Let $U_0=\{r\}$. For $j \geq 1$ given $U_{j}$, the set $U_{j+1}$ is obtained by
defining for each node $u \in U_j$, a subset of non-WTA nodes $V(u)$ as described next. \\
\textbf{Case 1: $u$ is a reset node.}
Set the depth of the subtree to be $d(u)=\min\{ DH-dist(u,r,T), DC\}$ and let $V(u)$ be the non-WTA nodes in the leaf nodes of $T_{d(u)}(u)$.  

Since $u$ is not a WTA node, it remains to consider the case where the number of active outputs in $Q(u)$ is at least $2$. Recall that $P_{01,j}(Q(u))$ be the probability of achieving WTA or reset in round $t+1$ given the configuration in round $t$ is $Q(u)$. We distinguish between two cases depending on the value of $P_{01,j}(Q(u))$.\\
\textbf{Case 2.1: $P_{01,j}(Q(u)) \geq \Theta(1/\log\log n)$.} Let $V'(u)$ be the children of $u$ in $T$ that are reset-nodes. For each reset-node $w \in V'(u)$, let $V(w)$ be the non-WTA nodes in the leaf nodes of $T_{d(u)-1}(w)$ and let $V(u)=\bigcup V(w)$. \\
\textbf{Case 2.2: $P_{01,j}(Q(u)) < \Theta(1/\log\log n)$.}  Let $V(u)$ be the children of $u$ that have at least $2$ active outputs in $Q(v)$ (hence $d(u)=1$).  
This completes the definition of $U_{j+1}$. 

To bound the weight of $U_{DH}$, we make use of the following claims that show that we do not loose too much weight in this traversal. Consider a node $u$ and let $NW_{DC}(u)$ be the set of non-WTA leaves of the tree $T_{DC}(u)$.
\begin{claim}
\label{cl:reset2}
If $u$ is a reset node, then $p_{node}(NW_{DC}(u))\geq c' \cdot p_{node}(u)$, for some constant $c'$. 
\end{claim}
\begin{proof}
Let $j$ be the layer of node $u$. 
Then by the selection of the initial configuration $C_0$, we know that the time it takes to achieve WTA with constant probability $c$ when starting from $C_0$ is strictly larger than $DC$. 
Since a reset node is labelled with this same initial configuration, we get that $p_{node}(NW_{DC}(u))\geq c' \cdot p_{node}(u)$ for $c'=1-c$. 
%
\end{proof}

\begin{claim}
\label{cl:type1}
Let $u$ be a node in layer $j$ that satisfies Case (1) or Case (2.1), then $p_{node}(V(u))\geq \Theta((1/\log\log n)^3) \cdot p_{node}(u)$. 
\end{claim}
\begin{proof}
If $u$ satisfies Case (1), the claim follows immediately by Cl. \ref{cl:reset2}. 
We now consider the case where $u$ satisfies Case (2.1). Recall that in this case the number of active outputs in $Q(u)$ is at least $2$. 
Let $A_0,A_1$ be the set of children of $u$ that are reset nodes, WTA nodes respectively. Let $A_{0,1} = A_0 \cup A_1$.

Then, since $u$ satisfies Case (2.1), $p_{node}(A_{0,1})\geq \Theta(1/\log\log n) \cdot p_{node}(u)$. In addition, since in $Q(u)$ there are at least two firing outputs, we can safely apply 
Cl. \ref{cl:probclose}, to have that $p_{node}(A_{0})\geq \Theta(1/(\log\log n)^2) \cdot p_{node}(A_{1})$. Combining these two inequalities, we get that 
$$p_{node}(A_{0})\geq \Theta(1/(\log\log n)^3) \cdot p_{node}(u).$$
Next, by using Cl. \ref{cl:reset2}, for every node $v \in A_0$ (which is a reset node), we have that $p_{node}(NW_{DC-1}(v))\geq c' \cdot p_{node}(v)$. All together, we get that 
\begin{eqnarray*}
p_{node}(NW_{DC}(u))&\geq& \sum_{v \in A_0} p_{node}(NW_{DC-1}(v)) 
\\&\geq&
c' \cdot p_{node}(A_0) \geq \Theta(1/(\log\log n)^3) \cdot p_{node}(u)~.
\end{eqnarray*}
Since $V(u)=A_0$, the claim follows.
\end{proof}

\begin{claim}
\label{cl:type2}
Let $u$ be a node that satisfies Case (2.2), then $\sum_{w \in V(u)}p_{node}(w)\geq 1-\Theta(1/\log\log n) \cdot p_{node}(u)$. 
\end{claim}
\begin{proof}
By the definition of $u$, $P_{01,j}(Q(u)) < \Theta(1/\log\log n)$. Hence, letting $V(u)$ be the children of $u$ that have at least $2$ active outputs in $Q(v)$ (hence $d(u)=1$), we have that $\sum_{w \in V(u)} \geq 1-\Theta(1/\log\log n) \cdot p_{node}(u)$. 
\end{proof}

Starting from a tree of weight $1$, we would like to show that at the end of the process after at most $DH$ iterations, the total weight of the leaf nodes $U_{DH}$ is at least $1/n^2$. 
%
%
%
%
%
%
%
We now use Cl. \ref{cl:type1} and \ref{cl:type2} to prove the lower bound. By Cl. \ref{cl:type1}, when we consider $u \in U_j$ that satisfies either case (1) or case (2.1), we keep $\Theta(1/(\log\log n)^3)$ fraction of the weight but enjoy a large jump of $DC$ layers in the sub-tree $T(u)$. Hence, on average, we keep $\Theta(1/(\log\log n)^3)^{1/DC}$ fraction of the weight of $T(u)$ per layer. 
By Cl. \ref{cl:type2}, in type (2), we keep at least $1-\Theta(1/\log\log n)$ fraction of the weight of $T(u)$ when moving from a node $u$ in layer $i$ to a subset of its children $V(u)$ in layer $i+1$. 
Hence, on average in every iteration, we keep at least 
$$\max\{(c/(\log\log n)^3)^{1/DC}, 1-c/\log\log n\}$$
fraction of the weight of the current forest. Hence, after $DH=DC \cdot \Theta(\log n/\log\log \log n)$ iterations, our total weight of the leaf set $U_{DH}$ is at least
$$(\max\{(c/(\log\log n)^3)^{1/DC}, 1-c/\log\log n\})^{DH}\geq 1/n^2.$$

\begin{figure}[h!]
\centering
\includegraphics[width=0.5\textwidth]{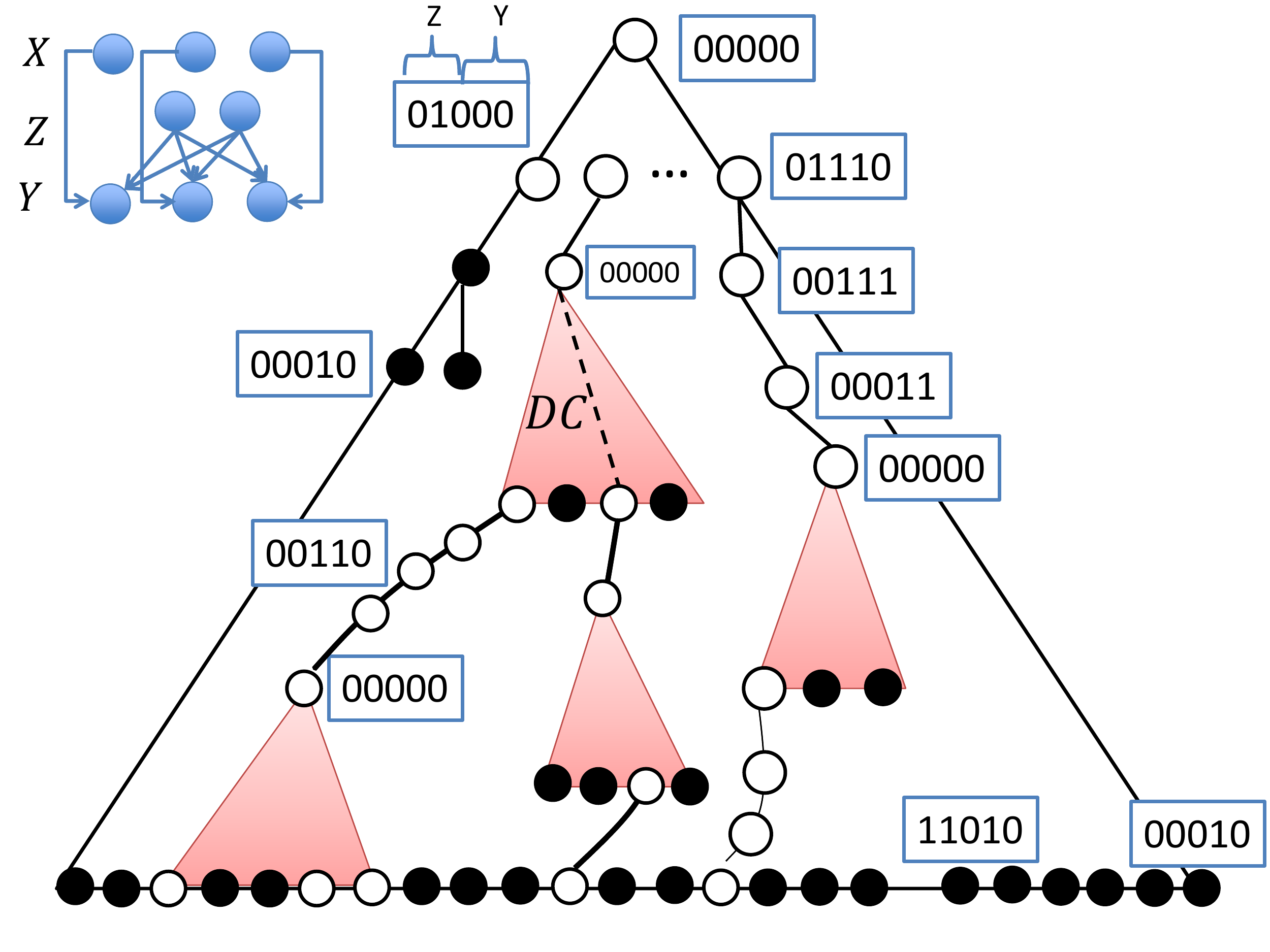}
\caption{The Execution Tree. Shown is a schematic illustration of the Execution Tree for a small network with two inhibitors and three outputs. Every node $u$ in layer $j$ is labeled by a vector of length $5$ describing an optional firing state for the inhibitors and outputs in round $j$. Each node has $2^5$ children -- covering all possible firing behaviors in round $j+1$. The weight nodes are non WTA nodes and the black nodes are the WTA nodes. When arriving a reset node $u$, a large jump is made by considering the leaf nodes of $T(u)$. When arriving a non-WR node, a small jump is made by considering subset of its children.}
\label{fig:simulationtree}
\end{figure}
}

\section{Discussion}\label{sec:discussion}

We hope that this paper is a starting point for further investigation into stochastic spiking networks from an algorithmic perspective,
which investigates fundamental tradeoffs between biological resources and identifies basic building blocks and principles for algorithm design in neural settings.

We focus on a restricted class of three layer networks, in which auxiliary neurons are not interconnected. This models the generally restricted connectivity structure that inhibitory neurons appear to have in biological networks and lets us give both very strong upper bounds and matching lower bounds. Still, it would be interesting to understand the effect of  connections between auxiliary neurons. We have preliminary work showing that some speedups are possible in these more general networks, however obtaining any non-trivial lower bounds would be very interesting.

Studying other important primitives aside from the binary version of WTA that we focus on would also be interesting. We again have preliminary work on \emph{non-binary WTA} in which the network must choose the input with the highest, or near highest firing rate as the winner. There are many other problems to consider.

Our model attempts to be biologically plausible enough to capture high level behavior, yet not be overly complex. However, many modeling assumptions are possible, and we hope that future work explores if changes to the model can lead to significant differences in computational power or algorithmic techniques.
As an example, for simplicity we considered a synchronous model, however, asynchrony seems to be an important part of neural computation which would be valuable to study.
%

Finally, we note that significant theoretical work attempts to understand how neural networks can \emph{learn} through the modification of synapse weights as their endpoints fire more or less frequently \cite{valiant2005memorization,papadimitrou2016cortical}. The most common model for how synapse weights evolve is the \emph{hebbian learning} rule, which is itself the focus of a vast literature. Merging the view of neural networks as executing algorithms given predetermined network parameters with understanding of learning would be very interesting. Can a WTA network `evolve' naturally via simple learning rules? How do fixed network motifs such as WTA circuits interact with more flexible `learning' networks?

\subsubsection*{Acknowledgments}
We are grateful to Mohsen Ghaffari for noting the general upper bound network construction and for many helpful discussions on the lower bound proof. We would also like to thank Nir Shavit, Rati Gelashvili, and Sergio Rajsbaum for insightful discussions.

\clearpage
\bibliography{wta}{}
\bibliographystyle{alpha}

\clearpage
\appendix

\section{Additional Discussion}
\subsection{Related Work}
\label{subsec:additionalrw}
\paragraph{Spiking Neural Network:} A vast literature studies computation in stochastic spiking neural networks. Work includes detailed models aimed at matching biological observations \cite{gerstner2002spiking,izhikevich2004model}, large scale simulation in hardware and software \cite{brette2007simulation,rossello2012hardware}, attempts to understand general properties of computation in these networks \cite{buesing2011neural}, the design of specific algorithms \cite{bohte2002error,seung2003learning}, and theoretical investigation of computational power \cite{maass1996computational,habenschuss2013stochastic}. For instance, it has been shown that deterministic spiking networks can simulate Turing machines and that stochastic spiking networks can implement MCMC sampling \cite{buesing2011neural}.
As is popular in the biologically-inspired algorithms literature, spiking networks have been used as heuristic `stochastic search' solvers for NP-hard constraint satisfaction problems, such as Sudoku and TSP \cite{jonke2016solving}.

Our model can be seen as a discrete version of the continuous model discussed in by Maass in \cite{maass2014noise} or as a noisy version of the deterministic model in \cite{maass1997networks}. In addition to being stochastic, in comparison to the model of \cite{maass1997networks}, our response latency $\Delta$ is constant for all connections in the network. Additionally, we have just a single round memory -- each neuron's membrane potential is affected just by spikes of neighboring neurons in the same or immediately preceding round of computation.  We note that if connections are allowed between auxiliary neurons, a longer memory can be easily be implemented within our general model.
 
%
\paragraph{Self-Stabilization in Distributed Computing:}
The notion of self-stabilization goes back to Dijkstra in 1973. A self-stabilizing system can automatically recover following the occurrence of transient faults. The goal in this area is to design systems that converge to a desired behavior from any arbitrary starting point \cite{dolev2000self,lynch1996distributed}. Among the tremendously broad work, perhaps the most relevant to this work is self-stabilizing algorithms for leader election \cite{dolev1997uniform,fischer2006self}.

In a stochastic neural network, self-stabilization is a necessity. Both changes to the given input as well as random deviations of the system from a converged state require the network to re-converge. Hence, we insure that all our networks converge to WTA from any initial network configuration and are self-stabilizing. This property does not hold in many previously studied WTA implementations for spiking networks \cite{oster2009computation}.

\paragraph{Valiant's Neuroidal Model:} Valiant considers a model of neural computation in which abstract neurons (which he calls \emph{neuroids}) are connected via a random network of synapses \cite{valiant2000circuits}. He discusses how these neurons can learn representations of real world objects whose perception stimulates the network in certain ways. As in our model, neurons fire in response to a membrane potential given by a weighted sum of firing neighbors. Differently, synapse weights evolve in response to increased firing of their end points, which allows \emph{learning} to occur within the network. This learning ability is the primary focus of Valiant's work and of follow up work on the model. For example, recently, \cite{papadimitriou2014unsupervised} extended understanding of how reasonably complex learning and pattern matching tasks can be performed in this model. 

Our work deviates is somewhat more `algorithmic' than the work of Valiant, focusing how basic takes can be computed using a set of neurons with a fixed set of synapses and bias values. We do not consider how, for example, our WTA networks could form within a larger neural circuit through learning of appropriate synapse weights. 
Following previous work \cite{maass1999neural} we think of WTA networks as fundamental primitives of neural circuits on top of which high level algorithms, such as learning algorithms, can be built. 

\subsection{Biological Motivation for Network Dynamics}
\label{append:bio}
\APPENDBIO

\section{Missing Proofs and Auxiliary Claims}\label{sec:missing}
Throughout, we make use of the following Corollary of the Chernoff bound.
\begin{theorem}[Simple Corollary of Chernoff Bound]
\label{thm:simplecor}
Suppose $X_1$, $X_2$, \dots, $X_\ell \in [0,1]$ are independent random variables. Let $X=\sum_{i=1}^{\ell} X_i$ and $\mu = \mathbb{E}[X]$. If $\mu \geq 5 \log n$, then w.h.p. $X \in \mu \pm \sqrt{5\mu\log n}$, and if $\mu < 5 \log n$, then w.h.p. $X \leq \mu +5\log n$.
\end{theorem}

\subsection{WTA with Two Inhibitors}
\label{append:twoi}
\APPENDTWO
\APPENDLOWEBOUNDTWO

\subsection{WTA with One Inhibitor}\label{appenx:oneinhib}
\dnsparagraph{One Inhibitor Lower Bound}
\begin{theorem}
\label{thm:lbonenotenough}
For any basic WTA network $\Net$ with $\NumInh = 1$ inhibitors, $\mathcal{ET}(\Net) = \Omega(n^c)$.
\end{theorem}

We fix any constant $c$ and assume by way of contradiction that there is a network $\Net$ which converges to WTA in $O(n^c)$ rounds in expectation. Let $z$ denote the single inhibitor in $\Net$.
We first argue that $\Net$ must be at least somewhat active -- given no firing activity from the outputs $\Output$ and the inhibitor $z$, each output connected to an active input should fire with reasonably high probability.
\begin{claim}[Sufficiently Active Network]\label{cl:notfire}
If $z^t = 0$ then each output $y_j$ with $x_j = 1$ and $y_j^t = 0$ fires in round $t+1$ with probability $\Omega(1/n^c)$. 
\end{claim}
\begin{proof}
Let $\Input$ be an input in which exactly one input $x_j$ fires and let $Y^0 = \vec{0}$. The time for $\Net$ to converge to WTA is lower bounded by the time required for $y_j$ to fire at least once.

Let $p_0$ be the probability that $y_j$ fires in round $t+1$ if $y_j^t = 0$ and $z^t = 0$ and let $p_1$ be the probability that $y_j$ fires in round $t+1$ if $y_j^t = 0$ and $z^t = 1$. $p_1 \le p_0$, so as long as $y_j$ does not fire in round $t$, it fires with probability at most $p_0$ in round $t+1$. If $p_0 \le c_1/n^c$ for some constant $c_1$ then starting from $C_0$, with constant probability, $y_j$ will not fire for $\Omega(n^c)$ consecutive rounds. By our assumption that $\Net$ converges to WTA in $O(n^c)$ rounds in expectation, we have $p_0 = \Omega(1/n^c)$.
\end{proof}

We next show that the inhibitor $z$ must fire in round $t$ \whp whenever at least one output fires, in order to maintain stability once WTA has been reached.
\begin{claim}[Stability]
\label{cl:atleastone}
For any configuration $\mathcal{C}^t$ of $\Net$, if at least one output neuron fires in round $t$ (i.e. $\norm{\Output^t}_1\ge 1$), $z$ fires in round $t$ \whp
\end{claim}
\begin{proof}
Consider input $\Input = \vec{1}$. 
Let $t$ be a round in which WTA is satisfied (exactly one output $y_j$ fires while no other outputs fire).
Using the notation of Claim \ref{cl:notfire}, the probability that a non-firing output fires in round $t+1$ is:
\begin{align*}
\Pr[z^{t} = 1 | \Output^{t}] \cdot p_1 + \Pr[z^{t} = 0 | \Output^{t}] \cdot p_0.
\end{align*}
By Claim \ref{cl:notfire} we have $p_0 \ge c_1/n^c$ for some constant $c_1$. 
Since $\Net$ converges to WTA it must be that \whp in round $t+1$, $y_j$ continues firing and no other output fires. So we have, for some large constant $c_2$:
\begin{align*}
\Pr[z^{t} = 1 | \Output^t] \cdot p_1 + \Pr[z^{t} = 0 | \Output^t] \cdot p_0 &\le 1/n^{c_2}\\
\Pr[z^{t} = 0 | \Output^t]  \cdot c_1/n^c &\le 1/n^{c_2}\\
\Pr[z^{t} = 0 | \Output^t]  &= O(1/n^{c_2-c})
\end{align*}
which gives the claim as long as $c < c_2$ since exactly one output fires in $Y^{t}$. The probability that $z$ fires when $> 1$ output fires is at least as large due to the excitatory nature of the outputs.
\end{proof}

Finally, by way of contradiction, we show that when $z$ fires, any output must stop firing with reasonably high probability. Otherwise, starting with multiple firing outputs, it will take too long to converge to WTA. As we will see this convergence requirement conflicts with the stability requirement of Claim \ref{cl:atleastone} since it means that the winning output will stop firing with reasonably high probability after convergence to WTA.

\begin{claim}[Convergence]
\label{cl:convergence}
If $z^t = 1$ then $y_j$ with $y_j^t = 1$ and $x_j = 1$ does not fire in round $t+1$ with probability $\Omega(1/n^c)$. 
\end{claim}
\begin{proof}
Let $p$ denote the probability that an output which corresponds to a firing input and which fires in round $t$ does not fire in round $t+1$ given that $z^{t} = 1$. We want to show that $p = \Omega(1/n^c)$.

Let $\Input = \vec{1}$ and let $t$ be any round in which at least two outputs fire. 
By Claim \ref{cl:atleastone}, $z^{t} = 1$ \whp and at least two outputs fire in round $1$ with probability $(1-p)^2 \ge 1-2p$. If we start from $Y^0 = \vec{1}$, then \whp at least two outputs will fire in $\Theta \left (\frac{1}{p} \right)$ consecutive rounds. By assumption $\Net$ converges to WTA within $O(n^c)$ rounds in expectation so we must have $p = \Omega(1/n^c)$.
%
%
\end{proof}

Putting it all together, consider an execution that satisfies WTA in round $t$ with exactly one output $y_j$ firing. Then, by Claim \ref{cl:atleastone}, $z$ fires in round $t$ \whp Thus, by Claim \ref{cl:convergence}, $y_j$ stops firing in round $t+1$ with probability $\Omega(1/n^c)$, in contradiction to the fact that the network must eventually converge to WTA and have $y_j$ fire for $n^{c_1}$ consecutive rounds for some large constant $c_1$.
We briefly note that the above lower bound can be matched with a trivial single inhibitor algorithm.
\begin{observation}
There is basic network $\Net$ with $\alpha =1$ inhibitors with $\mathcal{ET}(\Net) = O(n^c)$.
\end{observation}
\begin{proof}
The single inhibitor $z$ simply fires \whp in round $t$ whenever $\ge 1$ outputs fire in round $t$. The weights are set such that when $z^t = 1$ and $y_j^t = 1$, $y_j$ fires in round $t+1$ with probability $1/n^{c+1}$. If $z$ does not fire, any $y_j$ with $x_j =1$ fires \whp 

It is not hard to see that starting with any input, we will reach a round satisfying WTA within $O(n^c)$ rounds in expectation and after this round is reached, WTA will be maintained for $O(n^{c-1})$ additional rounds in expectation (and so $O(n^{c-2})$ w.h.p.).
\end{proof}

\subsection{WTA with $O(\log n)$ Inhibitors}\label{appenx:logn}
\begin{proof}[\textbf{Proof of Theorem \ref{theorem:logn} ($O(\log n)$ Inhibitor Upper Bound}]
Recall that we assume w.l.o.g. that $1/\lambda = c_1\log n$ for some constant $c_1$. We set $\weightX = 3$, $\weightS = 2$, and $\BiasOut = 3$. In this way, exactly as in the two inhibitor network analyzed in Section \ref{append:twoi}, any output $y_j$ with $x_j = 0$ will have $pot(y_j,t) \le \weightS - \BiasOut = -1$ in every round $t$ and so will not fire \whp in any round. 

Our network has $\alpha = \lceil \log n \rceil$ inhibitors. The first
is a stability inhibitor $z_s$, which behaves exactly as the stability inhibitor in the two inhibitor network analyzed in Section \ref{append:twoi}. $\weightY_s = 1$, $b(z_s) = 0.5$ and $\weightZ_s = -1$. $z_s$ fires \whp in sub-round $(t,3)$ if $\ge 1$ output fires in sub-round $(t,2)$ and does not fire \whp if no output fires.
We also have $\alpha-1$ convergence inhibitors $z_1,...,z_{\alpha-1}$. For each $z_i$, $b(z_i) = 2^{i} - .5$ and $\weightY_i = 1$. Therefore, $z_i$ fires \whp in round $t$ whenever $\ge 2^{i}$ outputs fire in the round. It does not fire \whp if $< 2^{i}$ outputs fire. We set the inhibitor weight from $z_1$ to each output to be $\weightZ_1 = -1$. For each $i \in 2,...,\alpha-1$ we set $\weightZ_i = -\lambda \cdot \log_2(e)$.

We can see that the stability Claim \ref{cl:stability} holds just as it does in the two inhibitor network analyzed in Section \ref{append:twoi}. Specifically, if just a single output $y_j$ with $x_j=1$ fires in some round $t$, \whp $z_s$ will fire while the convergence inhibitors will all not fire. So we will have:
\begin{align*}
pot(y_j,t+1) = \weightZ_s + \weightX + 1 \cdot \weightS - \BiasOut = -1 + 3 + 2 -3 = 1
\end{align*}
so $y_j$ fires \whp in round t+1. At the same time for $j' \neq j$, since $y_{j'}$ does not fire in round $t$:
\begin{align*}
pot(y_{j'},t+1) = \weightZ_s + \weightX + 0 \cdot \weightS - \BiasOut = -1 + 3 + 0 -3 = -1
\end{align*}
so $y_{j'}$ will not fire in round $t+1$. So, once a single $y_j$ with $x_j = 1$ fires in some round $t$, $\Net$ will converge to WTA \whp
We now show that $\Net$ reaches such a round in $O(1)$ expected time. 

Consider any round $t > 0$ in which $k_t \ge 2$ outputs fire. We can assume that all these outputs corresponding to firing inputs since as discussed, outputs corresponding to non-firing inputs do not fire \whp in any round. For some $i$ we have $k_t \in [2^i,2^{i+1})$ and so \whp in round $t$, $z_s,z_1,...,z_i$ fire while all other inhibitors do not fire (note that $\alpha -1 = \lceil \log n \rceil -1$ and so even if $n$ outputs fire, all inhibitors fire).
We thus have, \whp for any active output $y_j$ with $y_j^t = 1$ and $x_j = 1$:
\begin{align*}
pot(y_j,t+1) &= \weightS + \weightX - \BiasOut + \weightZ_s + \weightZ_1 + \sum_{j=2}^{i} \weightZ_j
\\&= 2 +3 -3 -1 - 1 - (i-1) \lambda = (i-1) \lambda \cdot \log_2(e).
\end{align*}
So $y_j$ fires in round $t+1$ with probability: 
\begin{align*}
p(y_j,t+1) = \frac{1}{1+ e^{(i-1) \lambda \log_2(e)/\lambda}} = \frac{1}{1+2^{i-1}}
\end{align*}
Since $k_t \in [2^i, 2^{i+1})$, we have $1 \le \frac{k_t}{1+2^{i-1}} \le 4$ and so can bound the probability that exactly one output that was active in round $t$ fires in round $t+1$ as:
\begin{align*}
k_t \cdot \frac{1}{1+2^{i-1}} \cdot \left (1-  \frac{1}{1+2^{i-1}} \right )^{k_t -1} &\ge \left (1-  \frac{1}{1+2^{i-1}} \right )^{k_t -1}\\
&\ge \left (1-  \frac{1}{1+2^{i-1}} \right )^{4(1+2^{i-1})}\\
&\ge \frac{1}{4^4}.
\end{align*}

So, with constant probability exactly one output that fired in round $t$ also fires in round $t+1$. Any output that did not fire in round $t$ has potential $\le \weightX - \BiasOut + \weightZ_s + \weightZ_\ell = -2$ and so does not fire with high probability. So, with constant probability, exactly one output $y_j$ with $x_j = 1$ fires, and so $\Net$ converges to WTA.

We conclude by noting that, by the arguments of Claim \ref{lem:live} for our two inhibitor network, with constant probability, starting with any $\Output^0$ we in fact have a round with $k_t \ge 1$ firing outputs all with active inputs within constant rounds. So from any starting configuration, we converge to WTA with constant probability in $O(1)$ rounds. Repeating this constant probability argument gives both $\mathcal{ET}(\Net) = O(1)$ and $\mathcal{HT}(\Net) = O(\log n)$.
\end{proof}

\dnsparagraph{$\Omega(\log n)$ High Probability Runtime Lower Bound}
\begin{theorem}
\label{thm:lowerbound_symmetric}
Any basic WTA network $\Net$, with any number of inhibitors, has $\mathcal{HT}(\Net) = \Omega(\log n)$.
\end{theorem}
\begin{proof}
We show that any network $\Net$ requires $\Omega(\log n)$ rounds before a round $t$ in which WTA is satisfied \whp This immediately gives our desired lower bound on convergence to WTA.

Consider input $X = \vec{1}$ (so any output is a valid winner) and  
any round $t$ such that WTA has not been satisfied for any $t' < t$. That is, in no round $t'$ does exactly one output $y_j$ fire. 
Let $W_t$ be the event that in round $t$ exactly one output fires and hence WTA is satisfied. We claim that $\Pr[W_t = 1~|~C^{t-1}] \le c$ for any configuration $C^{t-1}$ of $\Net$ in round $t-1$ and some universal constant $c$. That is, no matter the network configuration in round $t-1$, WTA will only be achieved with constant probability in the next round. 
Hence, as long as the initial output configuration $Y^0$ is one in which WTA is not satisfied, for $t = O(\log n)$, with probability at least $(1-c)^{t} = \Omega(1/n^{c'})$, for some constant $c'$, WTA will not be satisfied in any even round up to $t$. This gives that $\mathcal{HT}(\Net) = \Omega(\log n).$ 
There are two cases to work through: 
\paragraph{Network Reset:} $Y^{t-1} = \vec{0}$. In this case, no output fired in round $t-1$. Since all outputs are identical w.r.t their edge weights and bias values, conditioned on the behavior $\Inh^{t-1}$ of the inhibitors in round $t-1$, all outputs will fire independently with some fixed probability $p$ in round $t$. For any $p$ and any $n \ge 2$, the probability that exactly $1$ will fire in round $t$ is: $$\Pr[W_t = 1~|~C^{t-1}]  = n \cdot p(1-p)^{n-1} \le \frac{1}{2}.$$

\paragraph{No Reset:} $\norm{Y^{t-1}}_1 \ge 2$ -- i.e. there are at least 2 firing outputs in round $t-1$. Let $O_1$ be the set of firing outputs in round $t-1$ and $O_0$ be the set of non-firing outputs. Conditioned on $\Inh^{t-1}$, any output in $O_1$ fires independently with some probability $p_1$ in round $t$ and any output in $O_0$ fires with some probability $p_0$. Further, $p_0 \le p_1$ since the only difference in membrane potential between the neurons in $O_0$ and $O_1$ will be whether their excitatory self loop is active.

For $a \in \{0,1\}$ let $V_a$ be the event that exactly $1$ output from $O_a$ fires in round $t$. Clearly, $W_t \subseteq V_1 \cup V_0$.
For any $p_1$, $\Pr[V_1~|~C^{t-1}] = |O_1| \cdot p_1(1-p_1)^{|O_1|-1} \le 1/2$ since we have not reached WTA and so $|O_1| \ge 2$. 
If $|O_0| = 0$, then vacuously, $\Pr[V_0 ~| ~C^{t-1}] = 0$ and hence $\Pr[W_t~|~C^{t-1}] \le 1/2$.
Alternatively, If $|O_0| \ge 2$ then we also have $\Pr[V_0~|~C^{t-1}] \le 1/2$ and, since all outputs fire independently conditioned on $C^{t-1}$,
$$\Pr[W_t~|~C^{t-1}] \le 1- \Pr[\neg (V_1 \cup V_0)] \le 1- (1-1/2)^2 = 3/4.$$

Finally, if $|O_0| = 1$ either $p_0 \le 1/2$, in which case $\Pr[V_0~|~C^{t-1}] \le 1/2$ and we again have $\Pr[W_t~|~C^{t-1}] \le 3/4$ or $p_0 \ge 1/2$ in which case $p_1 \ge p_0 \ge 1/2$, and the probability that at least two outputs from $O_1$ fire is at least $1/4$ and hence WTA is achieved with probability at most $3/4$.
\end{proof}

\subsection{WTA with $\alpha \ge 2$ Inhibitors}\label{appenx:alpha}
\begin{proof}[\textbf{Proof of Theorem \ref{thm:upper_symmetricalpha}}]
We first describe the network construction in detail. 
As in our previous networks, we have a stability inhibitor $z_s$ that fires \whp whenever $\ge 1$ outputs fire in round $t$. This inhibitor ensures that in round $t+1$ \whp only outputs that fired in round $t$ (and hence have an active self loop) will fire in round $t+1$.

We set the excitatory input to output connection weight to $\weightX = 3$, the excitatory output self-loop to $\weightS = 2$, and the output bias to $\BiasOut = 3$.
For the stability inhibitor we set the excitatory output to inhibitor weight $\weightY_s 1$, $b(z_s) = .5$, and $\weightZ_s = -1$  just as we did in the two inhibitor algorithm. 

We have $\theta$ groups each containing $\lceil (\log n)^{1/\theta} \rceil $ convergence inhibitors, $Z_1,Z_2,...,Z_{\theta}$ where we denote $Z_i = \{z_{i,1},z_{i,2},...,z_{i,\lceil (\log n)^{1/\theta} \rceil} \}$.
We set $\weightY_i = 1$ for all $i \in Z_1, Z_2,...,Z_\theta$ and $b(z_{i,j}) = 2^{jd_{i}} - .5$.
In this way, when $k_{t} \in \left [2^{jd_{i}},2^{(j+1)d_i} \right)$ \whp $z_s,Z_1,...,Z_{i-1},z_{i,1},...,z_{i,j}$ all fire while the remaining inhibitors do not.
We set $\weightZ_{i,j}$ such that 
\begin{align}\label{bucketPots}
pot_{i,j} = \weightX + \weightS + \weightZ_s - \BiasOut + \sum_{\{(k,l) | k < i \text{ or } l \le j\}} \weightZ_{k,l}
\end{align} 
satisfies:
\begin{align}\label{bucketProbs}
p_{i,j} = \frac{1}{1+e^{-pot_{i,j}/\lambda}} = \frac{c_1}{2^{jd_i}}
\end{align}
for some small constant $c_1$.
For simplicity of presentation, we do not explicitly calculate out these weights. However, it is clear that choosing correct weights $p_{i,j}$ decreases as most inhibitors fire and the sigmoid function is continuous and decreases monotonically as $pot_i$ decreases.
We are now ready to analyze the network behavior in detail.

\paragraph{No Firing Inputs.}
As in the two inhibitor network,
any $y_j$ with $x_j = 0$, has maximum potential is $\weightS - \BiasOut = -1$ (even when no inhibitors fire) so and will not fire \whp outside of the initial configuration $Y^0$. ($p(y_j,t) \le \frac{1}{1+e^{1/\lambda}} \le 1/n^c$ for any $t$ since $\lambda = 1/c_1\log n$).
If $\Input = \vec{0}$, this implies that a valid WTA state in which no outputs fire will be converged to \whp trivially. 
We now focus on the case when $\norm{\Input}_1 \ge 1$. 

\paragraph{Maintaining WTA (Stability).} If just a single output $y_j$ corresponding to an active input ($x_j = 1)$ fires in round $t$ then \whp by Claim \ref{cl:stability} in Appendix \ref{append:twoi}, $\Net$ converges to WTA. This is because \whp just $z_s$ will fire in round $t$ and $y_j$ has potential $$pot(y_j,t+1) = (1 \cdot \weightZ_s) + (0 \cdot \weightZ_\ell) + (1 \cdot \weightS) + \weightX - \BiasOut = -1 + 2 + 3 - 3 = 1.$$
So $y_j$ fires with probability $\frac{1}{1+e^{c_1\log n}} \ge 1-1/n^c$ in round $t+1$. In contrast, for any $j' \neq j$, $y_{j'}$ does not fire in round $t$ so has
$$pot(y_{j'},t+1) \le (1 \cdot \weightZ_s) + (0 \cdot \weightZ_\ell) + (0 \cdot \weightS) + \weightX - \BiasOut = -1 + 3 - 3 = -1.$$
Therefore $y_j'$ fires with probability $\le 1/n^c$ in round $t+1$ so WTA is satisfied with output $y_j$ firing in round $t+1$ \whp

\paragraph{Converging to WTA.} It now just remains to show that with constant probability, within $O(\theta)$ rounds, there is at least one round in which exactly one output $y_j$ with $x_j^t =1$ fires. By the stability argument above once such a round occurs, $\Net$ will converge to WTA \whp

By the arguments of the convergence Claim \ref{lem:live} for the two inhibitor network, with constant probability, starting with any $\Output^0$ we in fact have a round with $k_t \ge 1$ firing outputs all with active inputs within constant rounds. If $k_t = 1$ then $\Net$ converges to WTA and we are done. So it suffices to consider the case when $k_t \ge 2$.

If $k_t \in \left [2^{jd_i},2^{(j+1)d_i} \right )$ then \whp $z_s,Z_1,...,Z_{i-1},z_{i,1},...,z_{i,j}$ fire while the other inhibitors do not and so in round $t+1$ any active output that fired in round $t$ fires with probability $p_{i,j}$. 
So we have $E [k_{t+1}] \in [1, c_12^{d_i})$, and, so with at least constant probability by a Markov bound
$k_{t+1} < 2^{d_i}$ if we set $c_1$ to a small constant.

 Additionally,
in any round with $k_t \ge 2$ conditioning on the high probability event that the correct inhibitors fire, $$\Pr[k_{t+1} = 1] = k_t \cdot p_{i,j} (1-p_{i,j})^{k_t -1}$$ and:
\begin{align*}
\Pr[k_{t+1} = 0] = (1-p_{i,j})^{k_t } &= \Pr[k_{t+1} = 1] \cdot \frac{(1-p_{i,j})}{k_t p_{i,j}}\\
&\le \Pr[k_{t+1} = 1] \cdot \frac{1}{2^{jd_i} \cdot c1/2^{jd_i}} \\
&\le \frac{1}{c_1} \Pr[k_{t+1} = 1].
\end{align*}

So, the probability of having exactly one output fire and hence converging to WTA is within a constant factor of the probability or having $0$ outputs fire and `reseting' the network.
So overall with constant probability, we reach such a round with $k_t =1$ within just $O(\theta)$ rounds.
Iterating this argument gives the expected and high probability runtime bounds of Theorem \ref{thm:upper_symmetricalpha}.
\end{proof}

\begin{proof}[\textbf{Proof of Theorem \ref{thm:upper_symmetrict}}]
Again we have a stability inhibitor $z_s$ that fires \whp in sub-round $(t,3)$ whenever $\ge 1$ outputs fire in sub-round $(t,2)$. We also have a `base level' convergence inhibitor that fires \whp whenever $\ge 2$ outputs fire. When just $z_s$ and $z_\ell$ fire in round $t$, any output (with an active input) that fired in round $t$ fires with probability $1/2$ in round $t+1$.

We then employ $\alpha -2$ additional convergence inhibitors $z_1,...z_{\alpha-2}$. For $i \in 1,...,\alpha-2$ let 
\begin{align*}
d_i = \left (\log n \right )^{i/(\alpha-1)}.
\end{align*}
Letting $k_t$ be the number of outputs that fire in round $t$, $z_i$ fires \whp in round $t$ whenever $ k_t \ge 2^{d_i}$.
The synapse weights from the inhibitors to the outputs are chosen such that, when $k_t \in \left [2^{d_i}, 2^{d_{i+1}} \right )$, and hence $z_1,...,z_i$ each active output (i.e. each $y_j$ with $y_j^t = 1$ and $x_j = 1$) fires with probability:
\begin{align*}
p_i = \frac{c \log n}{d_i} = \frac{c \log n}{\left (\log n \right )^{i/(\alpha-1)}}
\end{align*}
 in round $t+1$. This probability is enough to ensure that within few rounds, we will have $< 2^{d_i}$ active outputs. Specifically,
since $k_t \in \left [2^{d_i}, 2^{d_{i+1}} \right )$, for 
\begin{align*}
r = \frac{\log k_t}{ \log 1/p_i} \le \frac{(\log  n)^{(i+1)/(\alpha-1)}}{(\log n)^{i/(\alpha-1)} - \log(c\log n) } = O \left ( (\log n)^{1/(\alpha-1)} \right )
\end{align*}
with high probability, there will be a round $r' = O(r)$ with 
$k_{t+r'} \le 2^{d_i}$. 
At the same time, $p_i$ is large enough that \whp we will not overshoot WTA and have $0$ firing outputs in round $t+r'$. Even if $k_t = 2^{d_i}$ then we have $k_t \cdot p_i = c\log n$ and so, for large enough $c$, with high probability, by a Chernoff bound (Theorem \ref{thm:simplecor}) at least $O(\log n)$ outputs fire in round $t+1$.

Overall, within $O \left ( (\alpha-2) (\log n)^{1/(\alpha-1)}  \right )$ rounds, the number of active outputs falls within $[2, 2^{d_1}]$ \whp Once $k_t$ is in this range, just $z_s$ and $z_1$ fire \whp so our network is essentially identical to the two inhibitor network described in the previous section and analyzed in detail in Appendix \ref{append:twoi}.
We thus reach WTA with constant probability in $ \Theta ( \log 2^{d_1} ) = \Theta \left ( (\log n)^{1/(\alpha-1)} \right )$ additional rounds, giving our final runtime bound of $O \left (\alpha (\log n)^{1/(\alpha-1)}  \right )$.

We now formalize the above arguments.
Following our earlier constructions, we set the excitatory input to output connection weight to $\weightX = 3$, the excitatory output self-loop to $\weightS = 2$, and the output bias to $\BiasOut = 3$.
Set the excitatory output to inhibitor weights $\weightY_s = \weightY_\ell = 1$, $b(z_s) = .5$, $b(z_\ell) = 1.5$, and $\weightZ_\ell = \weightZ_s = -1$  just as we did in the two inhibitor algorithm. 

For the additional convergence inhibitors, set $\weightY_i = 1$ for all $i \in 1,...,\alpha-2$ and $b(z_i) = 2^{d_i} - .5$.
In this way, when $k_{t} < 2^{d_1}$, \whp just $z_s$ and $z_1$ fire, and each active output in round $t$ has potential 
$$pot(y_j,t+1) = \weightX + \weightS + \weightZ_s + \weightZ_\ell - \BiasOut = 3 + 2 -1 - 1 -3 = 0$$ 
and so fires with probability $p_1 = 1/2$ in round $t+1$.
We set $\weightZ_i$ such that 
\begin{align}\label{bucketPots}
pot_i = \weightX + \weightS + \weightZ_s + \weightZ_\ell - \BiasOut + \sum_{j=1}^i \weightZ_j
\end{align} 
satisfies:
\begin{align}\label{bucketProbs}
p_i = \frac{1}{1+e^{-pot_i/\lambda}} = \frac{c\log n}{2^{d_i}}.
\end{align}

As in the proof of Theorem \ref{thm:upper_symmetricalpha}, we do not explicitly calculate out these weights. Roughly, $\weightZ_i \approx \Theta(\frac{\lambda \log \log n}{\alpha-1})$ such that when $i$ inhibitors fire $p_i \approx \frac{1}{e^{-\Theta(\frac{i\lambda \log \log n}{\alpha-1})}} \approx \frac{c \log n}{2^{d_i}}$. It is clear that choosing correct weights is possible as $1/2 > p_1 > ... > p_{\alpha-1}$ and the sigmoid function is continuous and decreases monotonically as $pot_i$ decreases.

By identical arguments to those in the proof of Theorem \ref{thm:upper_symmetricalpha}, we converge to WTA in constant rounds \whp if there are no firing inputs or if a single output with a firing input fires in a round. Hence it just remains to show that with constant probability, within $O(\alpha (\log n)^{1/(\alpha-1)})$ rounds, there is at least one round in which exactly one output $y_j$ with $x_j^t =1$ fires.

Again, by the arguments of the convergence Claim \ref{lem:live} for the two inhibitor network, with constant probability, starting with any $\Output^0$ we in fact have a round with $k_t \ge 1$ firing outputs all with active inputs within constant rounds. If $k_t = 1$ then $\Net$ converges to WTA and we are done. So it suffices to consider the case when $k_t \ge 2$.
In this case, as discussed if $k_t \in \left [2,2^{d_1} \right )$ then \whp just $z_s$ and $z_\ell$ fire, and so each active output has potential
$$pot(y_j,t+1) = (1 \cdot \weightZ_s) + (1 \cdot \weightZ_\ell) + (1 \cdot \weightS) + \weightX - \BiasOut = -1 -1 + 2 + 3 - 3 = 0$$
and fires with probability $1/2$ in round $t+1$. 
All inactive outputs, which did not fire in round $t$, do not have an active self loop and hence have $pot(y_j,t) = -2$ and don't fire in round $t+1$ \whp (as discussed, all outputs with $x_j = 0$ also do not fire \whp)

Conditioning on this event, 
 with probability $1/2$, $k_{t+1} \le k_t/2$ and by the arguments in Claim \ref{lem:live}, we converge to WTA with constant probability within $O(k_t) = O(d_1) = O \left ((\log n)^{1/(\alpha-1)} \right)$ rounds.

If $k_t \in \left [2^{d_i},2^{d_{i+1}} \right)$ for some $i \in 1,...,\alpha-2$ then as discussed, \whp $z_s,z_\ell, z_1,...,z_i$ all fire in round $t$ while all other inhibitors do not fire. We thus have
\begin{align*}
\E[k_{t+1}] \ge 2^{d_i} \cdot p_i = \frac{2^{(\log n)^{i/(\alpha-1)}} \cdot c \log n}{2^{(\log n )^{i/(\alpha-1)}}} = c \log n
\end{align*}

By a Chernoff bound (Theorem \ref{thm:simplecor}), \whp $k_{t+1}$ falls within a constant multiplicative factor of its expection. Thus, \whp we still have $k_{t+1} \ge 2$. At the same time, \whp $k_{t+1} \le c_1 k_t \cdot p_i$ for some constant $c_1$. So overall, within $r = \frac{\log k_t}{ \log 1/p_i} = O \left ( (\log n)^{1/(\alpha-1)} \right )$ rounds, \whp $k_{t+r} < 2^{d_i}$.
Within $\alpha -2$ epochs of $O \left ( (\log n)^{1/(\alpha-1)} \right )$ rounds we thus have $k_t \in \left [2,2^{d_1} \right )$ \whp and then reach WTA withing $O \left ( (\log n)^{1/(\alpha-1)} \right )$ additional rounds with constant probability. 

Iterating this constant proability argument gives the expected and high probability runtime bounds of Theorem \ref{thm:upper_symmetrict}.
\end{proof}

\subsection{Missing Proofs for Main Lower Bound (Theorem \ref{thm:lbzaconst})}\label{appenx:lb}

\subsubsection{Inhibitors are Nearly Deterministic for Most Density Classes}\label{append:det}
\APPENDSAFTEYNOFIRE

\APPENDRNOFIRE

\APPENDDETL

\subsubsection{Detailed Description of the Prediction Process}
\label{Append:LBEDetailed}
\APPENDDetLBE

%
%
%
\paragraph{Monotonicity property of basic WTA networks.}
We show that the WTA dynamic is monotone so long as there is at least one firing output. Intuitively, we show that all basic WTA networks pick a single winner by monotonically decreasing the number of firing outputs until just a single output is firing. The number of firing outputs only ever increases if the network `overshoots' the WTA state and has a round in which no outputs fire.
\begin{lemma}
\label{lem:monotone} For any basic WTA network $\Net$, as long as the number of firing outputs is more than one, their number is monotone non-increasing. In particular, if at least one output fires in round $t$, \whp, an output that did not fire in that round, will not fire again in round $t+1$. 
\end{lemma}
\def\APPENDMONOTONE{
\begin{proof}
Given input $\Input$ with at least one firing input neuron, the network $\Net$  must eventually converge so that in every round exactly $1$ output fires \whp
Consider a round $t$ in this \emph{steady state period}. Since all outputs have the same parameters (e.g., edge weights and bias values) and since the weight of the self-loop is positive, if output $y_i$ fires in round $t$, it is at least as likely to fire in round $t+1$ as output $y_j$ for any $j\neq i$. Additionally, conditioned on the configuration of the inhibitors in time $t$, the probability that each output fires in round $t+1$ is independent.
Hence, it must be that w.h.p., if $y_i$ fired in round $t$, it continues to fire in round $t+1$ and each $y_j$, which did not fire in round $t$ does not fire in round $t+1$ with high probability. 

Further, consider any round $t$ with at least one firing output. Since all connections from the output layer are excitatory, the probability that any inhibitor in $\Inh$ fires at the end of round $t$ is at least as large as it is in the steady state of the network, and hence any output that does not fire in round $t$ does not fire in round $t+1$ \whp
\end{proof}
}

\APPENDMONOTONE

\subsection{Complete Description for High Probability Lower Bound (Lemma \ref{lem:highprobfromcprob})}\label{appenx:lbhp}
\APPENDHIGHLB




\section{Extension to Excitatory Auxiliary Neurons}
\label{sec:excitat}
In this section, we consider the more general case where the auxiliary neurons can be either excitatory or inhibitory. Let $\NumInh$ denote their number. We assume that outputs with no active input are not allowed to fire. Hence, in a given sub-round $(t,2)$, we consider two types of outputs that might fire: \emph{active} outputs -- those that fire in the previous round and hence have a positive feedback via the self-loop; and \emph{inactive} outputs -- those that did not fire in the previous round. Whereas in the inhibitory case, we could show that the dynamic is monotonic -- hence incative outputs do not fire with high probability, here it is not the case. Specifically, it might be the case that the level of inhibition during the process to achieve the WTA state is \emph{lower} than that in steady-state and hence inactive outputs (outputs that did not fire in the previous rounds) join the game in later rounds. 
In our lower bound proofs, we heavily used the monotonicity property as it allowed us focus only on the active outputs (those that fired in the previous rounds) and totally neglect the inactive ones. In this section,  we revise the claims that are based on the monotonicity lemma and adapt the proof to the general case of excitatory \& inhibitory neurons. 
\subsection{Extensions for the Lower Bound for Expected Time}
We classify the auxiliary neurons as before into three classes $S,C$ and $R$. Note that all the proofs that concern the predictability of the inhibitors, i.e., Lemmas \ref{lem:safetyinhibitors},\ref{lem:convergenceinhibitors},\ref{lem:almostdeter} depend only on the potential functions of the inhibitors and not on their effect on the outputs. Since the excitatory auxiliary neurons have exactly the same potential functions, the proofs follow immediately. 

The main adaptation is in the second part where we use the predictability of the auxiliary neurons to predict the network for at least one input configuration. 
We proceed by bounding the gap in potentials between active outputs and inactive outputs by showing that the weight of the self-loop is large. 
\begin{observation}
\label{obs:selflopp}
$\weightS\geq 2c\cdot \log n$.
\end{observation}
\begin{proof}
In the steady state situation, there exists one leader $u$ that fires in each round w.h.p.  $1-1/n^c$ for polynomially many rounds. On the other hand, all other outputs $v$ that do not have the positive feedback from the self-loop fire with probability $1/n^c$. Hence for such a round $t$ in steady state, we have: $pot_t(u)\geq c\log n$ and $pot_t(v)\leq -c\log n$. We get that $\weightS=pot_t(u)-pot(v)\geq 2c\log n$. The observation follows. 
\end{proof}
An immediate corollary of that is the following:
\begin{corollary}
\label{cor:notactiveactivehigh}
Consider a sub-round $(t,2)$ and let $F_{t-1}$ be the firing configuration of the auxiliary neurons in sub-round $(t-1,3)$. If the firing probability of an inactive output $v$ (output that did not fire in the previous sub-round $(t-1,2)$) in sub-round $(t,2)$ is at least $1/n^{c}$, then the firing probability of an active output $u$ in sub-round $(t,2)$ is $\geq 1-1/n^{c}$.
\end{corollary}
\begin{proof}
Since all outputs have the same connections to the auxiliary neurons, only difference in the potential of an inactive output and an active output is the weight of the self-loop. Hence, $pot_t(u)=pot_t(v)+\weightS\geq -c\log n+2c\log n\geq c\log n$, where the first inequality follows by plugging Obs. \ref{obs:selflopp} and using the fact that the firing probability of $v$ is $1/(1+e^{-pot_t(v)})\geq 1/n^c$. Thus, $u$ fires with probability $1/(1+e^{-c\log n})=1-1/n^c$. 
\end{proof}
We now turn to consider the second part of the lower bound where we predict $\Omega(\log\log n/\log \NumInh)$ rounds of the network for at least one density input class. 
Since in the zero round no-output fires and w.h.p. also no auxiliary neuron is firing (since their bias value is $\omega(\log n)$), predicting the number of firing outputs in round $1$ is exactly the same as in the only-inhibitor case.
\paragraph{Predicting the number of firing outputs in round $t \ge 2$:}
We first  define a subset of inputs $\mathcal{X}^{large}_t \subseteq \mathcal{X}_{t-1}$ for which we can predict the behavior of the outputs in the network $\Net$ in round $t$.
Let $\mathcal{X}^{same}_t\subseteq \mathcal{X}_{t-1}$ be the largest subset of inputs whose predicted firing vector $F_{t-1}(\Input)$ for the auxiliary neurons in round $t-1$ is the \emph{same}, and denote this common firing vector by $F^*_{t-1}$. Let $\mathcal{X}^{large}_t$ be the set of inputs in $\mathcal{X}^{same}_t$ after omitting $\Theta(\log\log n)$ inputs with the smallest range value in round $t-1$.
Eventually we will show that $\mathcal{X}^{large}_t$ is a reasonably large set of inputs compared to $\mathcal{X}_{t-1}$, and hence we can continue predicting behavior for at least some inputs for a large number of rounds. But first we show how to  predict $R_{t}(\Input)$ for 
every input $\Input \in \mathcal{X}^{large}_t$.

Let $p'$ be the firing probability that an inactive output (one with $y_j^{t-1} = 0$) fires in sub-round $(t,2)$ given that the inhibitors fired in sub-round $(t-1,3)$ according to $F^*_{t-1}$. Since all inputs in $\mathcal{X}^{same}_{t}$ have the same predicted firing vector $F^*_{t-1}$, in each of them, an inactive output 
fires in sub-round $(t,2)$ with probability $p'$. Let $p$ be the corresponding firing probability of an active output. 
We now consider two cases depending on the value of $p'$. If $p' < 1/n^c$, we predict that no inactive output fires in that round. Note that this prediction holds with probability $\geq 1-1/n^{c-1}$.
In such a case we only predict the range for the active outputs in the exact same manner as before. Note that when we predicted the range of firing active outputs in the previous section, we did not use the fact that the auxiliary neurons are inhibitory, only that all competing outputs whose cardinality is to be estimated fire with the same probability in that round. 

Next, we consider the more interesting case where $p'\geq 1/n^c$, that is the inactive outputs have a fair chance of firing in sub-round $(t,2)$. Here, we make use of Lemma \ref{cor:notactiveactivehigh} that says that with probability at least $1-1/n^{c-1}$, all active outputs (i.e., that fired in round $t-1$) fire in sub-round $(t,2)$ as well. Let $k=2^i$ be the number of active inputs in the vector $\Input$.
Let $E_{t-1}=E(\widehat{R}_{t-1}(\Input) ~\mid~ F_{t-2}(X))$ be the expected number of firing outputs in sub-round $(t-1,2)$ given the predicted firing vector $F_{t-2}(X)$.
Then, the expected number of firing outputs in sub-round $(t,2)$ is 
$$E(\widehat{R}_{t}(\Input) ~\mid~ F_{t-1}(X))=E_{t-1}+p' \cdot (k-E_{t-1})=(1-p')\cdot E_{t-1}+p'\cdot k.$$
\begin{claim}
\label{cl:expectedseparated}
Let $\Input_1,\Input_2 \in \mathcal{X}_t$ be such that $||\Input_1||_1 \geq 2||\Input_2||_1$. Then $E(\widehat{R}_{t}(\Input_1) ~\mid~ F_{t-1}(\Input_1)) \geq 2E(\widehat{R}_{t}(\Input_2) ~\mid~ F_{t-1}(\Input_2))$.
\end{claim}
\begin{proof}
We will prove by induction on the number of rounds $t$. Let $k_j =||\Input_j||_1$ and
$E_{j,t}=E(\widehat{R}_{t}(\Input_j) ~\mid~ F_{t-1}(\Input_j))$ for $j \in \{1,2\}$.

Since $\Input_1,\Input_2 \in \mathcal{X}_t$, it holds that $\Input_1,\Input_2 \in \mathcal{X}_\ell$ for every $\ell \in \{1, \ldots, t\}$ hence $F_{\ell}(\Input_1)=F_{\ell}(\Input_2)$ for every $\ell \in \{1, \ldots, t\}$. For the base of the induction of round $t=1$, this clearly holds since $E_{0,t}=p_0 \cdot k_j$, $j \in \{1,2\}$, where $p_0$ is the firing probability of an output where in the previous round no one fired. 
Assume the claim holds up to round $t-1$. We have that $E_{j,t}=E_{j,t-1}+p' \cdot (k_j-E_{j,t-1})=(1-p')E_{j,t-1} +p'k_j$, for $j \in \{1,2\}$. By the induction assumption for $t-1$, we get $E_{1,t-1}\geq 2 \cdot E_{2,t-1}$ and by definition $k_1 \geq 2 \cdot k_2$, overall 
$E_{1,t}\geq 2E_{2,t}$ as required. 
\end{proof}
We get that the expected number of firing outputs (conditioned on the predictions) are $2$-separated. Now, we can claim exactly as before that all these expected values should be $\Omega(1/\log^4 n)$ as otherwise there is at least one input configuration for which there is a reset (i.e., in the next round no output fires) for $\Omega(\log n)$ times (see Obs. \ref{Obs:resetorcontinue}). 

Since all expected predictions for the number of firing outputs are $\Omega(1/\log^4 n)$, by removing the $\Theta(\log\log n)$ inputs from $\mathcal{X}^{same}$ (i.e., as given by set $\mathcal{X}^{large}$), we get that all expected numbers of firing outputs are $\Omega(\log^7 n)$ and hence the random variables $\widehat{R}_{t}(\Input)$ are well concentrated around their expectation. The remaining proof goes exactly the same as in the inhibitory-case. 

\subsection{Extensions for the Lower Bound for High Probability Time}
We define the \emph{weak} WTA state to be state in which exactly one active output is firing (but possibly many inactive firing outputs). Whenever we use the notion of WTA nodes in the proof of Lemma \ref{lem:highprobfromcprob}, we now use the notion of weak WTA nodes instead. The definition of a reset node remains as is, i.e., a node $u$ such that in its configuration $Q(u)$ no output (of any type) fires. 

Note that the lower bound proof for the expected time implies that there is an input $X_0$ such that with a good probability after $t=\Omega(\log\log n/\log \NumInh)$ rounds there are still $\Omega(\log n)$ competing outputs. After $t+1$ rounds, either we can assume w.h.p. that no inactive output fires or that all the $\Omega(\log n)$ active outputs fire. Hence, the lower bound implies that after $t+1$ rounds, with good probability, the number of firing active outputs is $\Omega(\log n)$, implying that the network is in a \emph{weak} WTA state.
Let $P_{1,j}(Q)$ be the probability that exactly one \emph{active} output fires in sub-round $(j,2)$ given that the auxiliary neurons fire in round $j-1$ according to $Q$. Similarly, let $P_{0,j}(Q)$ be the probability that \emph{no active output} fires in sub-round $(j,2)$ given $Q$. Finally, let $P_{01,j}(Q)$ be the probability that at most one active output fires in round $(j,2)$ given that configuration in round $j-1$ is $Q$, hence, $P_{01,j}(Q)=P_{1,j}(Q)+P_{0,j}(Q)$. Since we consider only the active outputs, Cl. \ref{cl:probclose} follows as is. 
We now claim the following.
\begin{corollary}
\label{cor:allresetexc}
For every round $j$ and for every vector $Q \in \{0,1\}^{n+\NumInh}$ in which there are at least two active firing outputs and such that $P_{01,j}(Q)\geq \Theta(1/\log\log n)$, it holds that there is a (total) reset in round $j$ (i.e., no output fires) with probability at least $\Theta(1/(\log\log n)^3)$. 
\end{corollary}
\begin{proof}
Since in $Q$ there are at least two firing \emph{active} outputs, by Cl. \ref{cl:probclose}, $P_{0,j}(Q)\geq \Theta(1/(\log\log n)^3)$. Hence the probability that no active output fires is at least $\Theta(1/(\log\log n)^3)$. We now claim that the probability that also no inactive output fires is at least $1-1/n^{c-1}$. Hence, by the independence between the output decisions (given the firing states of the inhibitors), we get that the probability that no output fires is at least $\Theta(1/(\log\log n)^3)$ as required.

Assume towards contradiction that inactive output fires with probability $\geq 1/n^c$. 
By Cor. \ref{cor:notactiveactivehigh}, we get that an active output fires with probability at least $1-1/n^c$. Since in the previous round there are at least two firing \emph{active} outputs, we get that with probability $\geq 1-1/n^c$ there are at least two firing outputs in sub-round $(j,2)$, contradiction to the assumption that $P_{01,j}(Q)\geq \Theta(1/\log\log n)$.

Thus we get that each inactive output fires with probability $<1/n^c$, and with probability $\geq 1-1/n^c$ no inactive output fires. The claim follows.
\end{proof}
Equipped with Cor. \ref{cor:allresetexc} and the lower bound for expected time, we can now use the execution tree to show that the weight of non weak-WTA nodes is at least $1/n^2$. 
The same idea generally holds up to few adaptations. 
Recall that in our execution tree traversal, at step $j$ we obtain a collection of non weak WTA nodes. That is nodes $u$ with configuration $Q(u)$ which either there are at least two active outputs that are firing.
For $j \geq 1$ given $U_{j}$, the set $U_{j+1}$ is obtained by
defining for each node $u \in U_j$, a subset of non weak WTA nodes $V(u)$ as described next. \\
\textbf{Case 1: $u$ is a reset node.}
Set the subtree depth $d(u)=\min\{ DH-dist(u,r,T), DC\}$ and let $V(u)$ be the non weak WTA nodes in the leaf nodes of $T_{d(u)}(u)$.  By the lower bound proof, the set $V(u)$ captures $1-1/\log n$ of the probability mass in $T(u)$. 

Since $u$ is a non weak WTA node, it remains to consider the case where the number of active firing outputs in $Q(u)$ is at least $2$. Recall that $P_{0,1}(Q(u))$ is the probability that in sub-round $(j,2)$ at most one active output fires given that the configuration in round $j-1$ is $Q(u)$. \\
\textbf{Case 2.1: $P_{0,1}(Q(u)) \geq \Theta(1/\log\log n)$.} 
Let $V'(u)$ be the children of $u$ in $T$ that are reset-nodes. For each reset-node $w \in V'(u)$, let $V(w)$ be the non-WTA nodes in the leaf nodes of $T_{d(u)-1}(w)$ and let $V(u)=\bigcup V(w)$. \\
By Cl. \ref{cor:allresetexc}, since the number of firing active outputs in $Q(u)$ is at least $2$ and since $P_{0,1}(Q(u)) \geq \Theta(1/\log\log n)$, the probability for a (total) reset in the next round is at least $\Theta(1/(\log\log n)^3)$ and hence $V'(u)$ captures  $\Theta(1/(\log\log n)^3)$ of the probability mass in $T(u)$. This will allow us to follow the same argument as before when following the case 2.1. \\
\textbf{Case 2.2: $P_{0,1}(Q(u))<\Theta(1/\log\log n)$.} 
Let $V(u)$ be the children of $u$ that have at least $2$ \emph{active} outputs in $Q(v)$ (hence $d(u)=1$). 
Since  $P_{0,1}(u)\leq \Theta(1/\log\log n)$, we capture $1-\Theta(1/\log\log n)$ of the weight of the tree $T(u)$. 
This completes the definition of $U_{j+1}$. The argument that uses this case follows now the exact same line. In sum, either we capture only $\Theta(1/\log\log n)$ of the probability mass in such a case we have a large jump in the tree or that we capture $1-\Theta(1/\log\log n)$ of the probability mass. As before using the amortization argument, overall the number fo non weak WTA can be bounded by $\geq 1/n^2$. The completes the extension to excitatory auxiliary neurons.

\end{document}